\documentclass{article}

\usepackage{fullpage}       %
\usepackage[sort]{natbib}
\usepackage[utf8]{inputenc} %
\usepackage[T1]{fontenc}    %
\usepackage[backref=page]{hyperref}
 \renewcommand*{\backrefalt}[4]{%
    \ifcase #1%
     \or (Cited on page~#2.)%
     \else (Cited on pages~#2.)%
    \fi%
    }

\usepackage{url}            %
\usepackage{booktabs}       %
\usepackage{amsfonts}       %
\usepackage{dsfont}
\usepackage{nicefrac}       %
\usepackage{microtype}      %
\usepackage{graphicx}
\usepackage{algorithm}
\usepackage{algorithmic}
\usepackage{kantlipsum}
\usepackage[]{mdframed}
\usepackage[inline]{enumitem}
\usepackage{bm}

\usepackage[theorem, hyperstyle, subfigref]{xiangpackage}

\usepackage{cleveref}
\usepackage{booktabs}
\usepackage{tabularx}
\usepackage{physics}
\usepackage{makecell}
\usepackage{tablefootnote}

\usepackage{crossreftools}
\usepackage{tikz}
\usetikzlibrary{tikzmark,calc}

\crefname{assumption}{assumption}{assumptions}
\crefname{desideratum}{desideratum}{desiderata}
\crefname{principle}{principle}{Principle}

\usepackage{enumitem}

\newif\ifsubmissionready 
\submissionreadytrue

\newcommand{\pdata}{p_{\mathrm{data}}}
\newcommand{\mudata}{\mu_{\mathrm{data}}}
\newcommand{\pprior}{p_{\mathrm{prior}}}
\DeclareMathOperator{\Pm}{P_{\cM}}

\title{When Scores Learn Geometry: Rate \\ Separations under the Manifold Hypothesis}

\addtocounter{footnote}{1} %
\author{
Xiang Li\thanks{Department of Computer Science, ETH Zurich, Switzerland.
\texttt{xiang.li@inf.ethz.ch},
\texttt{zebang.shen@inf.ethz.ch},
\texttt{yaping.hsieh@inf.ethz.ch},
\texttt{niao.he@inf.ethz.ch}.}
\and
Zebang Shen\footnotemark[2]
\and
Ya-Ping Hsieh\footnotemark[2]
\and
Niao He\footnotemark[2]
}

\date{}

\begin{document}

\maketitle

\begin{abstract}
Score-based methods, such as diffusion models and Bayesian inverse problems, are often interpreted as learning the \textbf{data distribution} in the low-noise limit ($\sigma \to 0$).
In this work, we propose an alternative perspective: their success
arises from implicitly learning the \textbf{data manifold} rather than the full distribution. Our claim is based on a novel analysis of scores in the small-$\sigma$ regime that reveals a sharp \textbf{separation of scales}: \emph{information about the data manifold is $\Theta(\sigma^{-2})$ stronger than information about the distribution.} We argue that this insight suggests a paradigm shift from the less practical goal of distributional learning to the more attainable task of \textbf{geometric learning}, which provably tolerates $O(\sigma^{-2})$ larger errors in score approximation. We illustrate this perspective through three consequences:
\begin{enumerate*}[label=\roman*)]
\item in diffusion models, concentration on data support can be achieved with a score error of $o(\sigma^{-2})$, whereas recovering the specific data distribution requires a much stricter $o(1)$ error;
\item more surprisingly, learning the \textbf{uniform distribution} on the manifold—an especially structured and useful object—is also $O(\sigma^{-2})$ easier; and
\item in Bayesian inverse problems, the \textbf{maximum entropy prior} is $O(\sigma^{-2})$ more robust to score errors than generic priors.
\end{enumerate*}
Finally, we validate our theoretical findings with preliminary experiments on large-scale models, including Stable Diffusion.
  \end{abstract}

\section{Introduction}

\emph{Score learning} has emerged as a particularly powerful paradigm for
modeling complex probabilistic distributions, driving breakthroughs in
generative modeling, Bayesian inverse problems, and
sampling~\citep{laumont2022bayesian,
saremi2023universal,ho2020denoising,song2019generative,songscore}. Let
$\mu_{\mathrm{data}}$ be a data measure over $\mathbb R^d$ and
define a Gaussian-smoothed measure as
\begin{equation} \label{eqn_p_sigma_definition}
    \mu_\sigma := \operatorname{law}\left(X + \sigma Z\right)
    \; \text{or} \;
    \mu_\sigma := \operatorname{law}\left(\sqrt{1-\sigma^2}\,X + \sigma Z\right), \; \text{where} \;
     X\sim\mu_{\mathrm{data}}, Z\sim\cN(0,I).
\end{equation}
Let $p_\sigma$ be its density function w.r.t. the Lebesgue measure over $\mathbb
R^d$. A key step in the score learning framework is to approximate the score
function $\nabla \log p_\sigma$ and to sample from the target distribution
$\mu_{\sigma}$, possibly across a spectrum of different $\sigma$
values~\citep{vincent2011connection,hyvarinen2005estimation}.

A central challenge in this framework is understanding the \emph{low-temperature limit}, i.e., learning the score of $\mu_\sigma$ as $\sigma \to 0$, which encodes the most detailed information about the data distribution. Empirically, this regime is also the most valuable: low-temperature scores underpin many probabilistic learning frameworks~\citep{laumont2022bayesian,saremi2023universal,janati2024divide,kadkhodaie2020solving}, including the influential diffusion model framework~\citep{ho2020denoising,song2020denoising,karras2022elucidating}, whose noise schedules are specifically designed to emphasize low temperatures and often require substantial post-training engineering to stabilize the learned scores. %

Despite its importance, accurately estimating the score function in the low-$\sigma$ regime remains notoriously difficult~\citep{songscore,karras2022elucidating,arts2023two,raja2025action,stanczuk2024diffusion}. Motivated by this challenge, this paper establishes a new qualitative phenomenon under the widely adopted \emph{manifold hypothesis}, which posits that the data distribution $\mu_{\mathrm{data}}$ is supported on a low-dimensional manifold $\cM$ embedded in a high-dimensional ambient space.

\begin{figure}[t]
\centering
\begin{subfigure}[t]{0.46\textwidth}
  \centering
  \includegraphics[width=\textwidth]{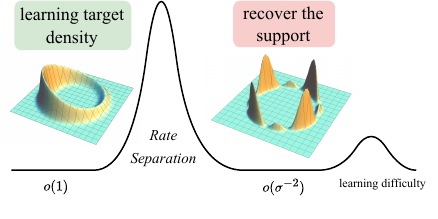}
  \caption{Existing Paradigm}
  \label{fig:existing_paradigm}
\end{subfigure}
\hspace{1em}
\begin{subfigure}[t]{0.48\textwidth}
  \centering
  \includegraphics[width=\textwidth]{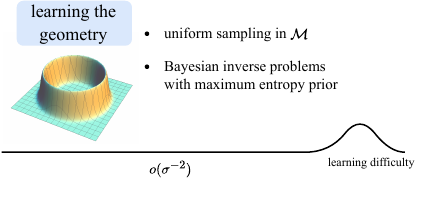}
  \caption{New Paradigm}
  \label{fig:new_paradigm}
\end{subfigure}
\caption{
Toy examples illustrating recovered distributions under different regimes, with the manifold represented as a one-dimensional circle embedded in $\mathbb{R}^2$.
}
\label{fig:overview}
\end{figure}

Our key finding, formalized in \Cref{thm_main_informal}, is that in the
small-$\sigma$ regime of score learning there is a \textbf{sharp separation of
scales}: \emph{geometric information about the data manifold appears at order
$\Theta(\sigma^{-2})$, whereas density information of $\mu_{\mathrm{data}}$
emerges only at order $\Theta(1)$}. As shown in \Cref{sec:Gaussian-smoothing},
this implies that distribution learning of $\mu_\sigma$ (e.g., in diffusion
models) \textbf{necessarily} first recovers the support of the data distribution before
any information about the density can be learned.
This perspective naturally separates score learning into two fundamental tasks:
\emph{geometric learning}, which targets the manifold geometry, and \emph{density
learning}, which targets the specific data density on that manifold, with the
latter being order of magnitude more difficult. It also suggests that the practical
success of score-based models (e.g., diffusion models) stems from
constraining generated samples to the manifold, thereby producing realistic data
even without fully recovering the underlying distribution.
According to our analysis, to achieve this, a score error even as large as $o(\sigma^{-2})$ is sufficient.

However, our analysis reveals a critical limitation: unless the score is
learned to a stringent accuracy that is beyond $O(1)$, attempts to recover the data distribution
may yield \emph{arbitrary} densities supported on the manifold. This amounts to
only a partial recovery of geometry and can compromise the reliability of
downstream tasks and analyses. Such an observation motivates us to pursue
\emph{full geometric learning}---that is, learning to sample
\emph{uniformly} with respect to the manifold’s intrinsic (Riemannian) volume
measure, as it is well-known that uniform samples can best support tasks that depend solely on
the underlying geometry (e.g., Laplace--Beltrami and heat-kernel approximation,
geodesic and diffusion distances)~\citep{coifman2006diffusion,
belkin2008towards, jost2005riemannian}. In addition, they also facilitate principled
manifold exploration, yielding diverse samples while mitigating potential biases present
in $\mudata$~\citep{de2025provable}.

In this light, a central contribution of this work is to show that a simple,
one-line modification to standard algorithms can \emph{provably} generate the
\emph{uniform distribution} on the manifold---requiring only
$o(\sigma^{-2})$ score accuracy, in stark contrast to the $o(1)$ accuracy
needed for exact distributional recovery. In summary, we advocate a paradigm
shift: from the demanding goal of \emph{distributional learning} toward the
more practical and robust objective of \emph{geometric learning}.

We substantiate the aforementioned rate separation phenomenon by three key results (see also \Cref{fig:overview}):
\begin{itemize}
    \item \Cref{thm:recover_pdata} shows that, in existing frameworks, the score accuracy required to force concentration on the data manifold is $O(\sigma^{-2})$ weaker than that needed to exactly recover $\mu_{\mathrm{data}}$. Nevertheless, the resulting distribution can still be \emph{arbitrary}.
    \item In contrast, \Crefrange{thm:recover_uniform_gradient}{thm:recover_uniform} establish a new paradigm centered on extracting precise \emph{geometric} information of the data manifold by producing the \emph{uniform distribution}. Notably, we show that a simple one-line modification of a widely used sampling algorithm suffices to obtain samples from the uniform distribution under the relaxed score error condition $o(\sigma^{-2})$, substantially weaker than the $o(1)$ required for full recovery of $\mu_{\mathrm{data}}$.
    \item 
        In the context of Bayesian inverse problems~\citep{venkatakrishnan2013plug}, \Cref{thm:with_guidence} establishes a rate separation in posterior sampling depending on the choice of prior. When the prior is uniform, posterior sampling requires only $o(\sigma^{-2})$ score accuracy. By contrast, when the prior is taken to be the commonly used data distribution $\mu_{\mathrm{data}}$, substantially stronger accuracy guarantees are needed to ensure provable success in existing works \citep{laumont2022bayesian,pesme2025map}. 
    
\end{itemize}

We validate these theoretical results with preliminary experiments on both synthetic and real-world data, including an application of our algorithm to a large-scale image generation model (Stable Diffusion 1.5~\citep{Rombach_2022_CVPR}).

\subsection{Related Work}
\label{app:related}

\paragraph{Diffusion models for distribution learning.}
Prior theory shows that diffusion/score-based samplers converge to the target
law when the learned score is accurate, with error bounds that scale directly
with the score mismatch~\citep{de2022convergence,chen2023sampling,lee2023convergence};
related works study other factors such as dimension dependence~\citep{azangulov2024convergence,tang2024adaptivity}.
However, these results do not separate geometry from density
in the score error but instead consider them together, therefore they do not imply any
scale separation.

\paragraph{Diffusion models detect data manifold.}
There is a growing body of work probing whether diffusion models learn the full
data distribution or primarily the underlying low-dimensional manifold. A number
of studies suggest that these models often capture the data support while
missing fine-grained distributional structure. However, these results are obtained
under restricted settings: \citet{stanczuk2024diffusion} focuses on estimating
the intrinsic dimension of the data manifold; \citet{ventura2024manifolds}
analyzes only linear manifolds (linear subspaces); and
\citet{pavlova2025diffusion} provides primarily empirical evidence.
\citet{pidstrigach2022score} establishes sufficient regularity conditions under
which high-accuracy scores concentrate mass near the manifold, but does not
address how approximation errors scale with $\sigma$ and therefore does not
reveal a separation of scales. 
By contrast, our analysis quantifies how inaccuracies in the learned score
propagate differently to geometry versus distribution learning,
exhibiting distinct error rates that lead to a sharp scale separation in the
small-$\sigma$ regime. Furthermore, prior work does not address full
geometric recovery via uniform sampling.

\paragraph{Asymptotic behavior of the score.}
It is established that under the manifold hypothesis, the score function develops a singularity in the small-noise regime, becoming orthogonal to the data manifold.  Recent works characterize this behavior mathematically, showing that the score effectively acts as a geometric projection operator onto the manifold~\citep{lu2023mathematical,lyu2025resolving,liu2025improving}. This aligns with the leading-order term in our expansion (\Cref{eq:score_expansion}), which governs geometric concentration. However, these analyses generally subsume the distributional information into a generic bounded remainder term. Crucially, they do not explicitly isolate the higher-order terms involving $\pdata$ and thus do not characterize the separation between geometry and density. Our analysis reveals that these missing terms are not merely residuals but are essential for establishing the rate separation between recovering the manifold support and learning the underlying density.

\paragraph{Uniform sampling on manifolds.}
Classical approaches achieve uniform-on-manifold sampling via graph-based
normalizations that cancel the sampling density so that the limiting operator is
the Laplace–Beltrami operator~\citep{coifman2006diffusion,hein2007graph}.
While foundational, these methods are designed to approximate geometric
operators from neighborhood graphs and do not readily scale to
high-dimensional, large-scale generative modeling.
Recently, \citet{de2025provable} proposed fine-tuning diffusion models to
produce uniform samples. In contrast, our approach operates entirely at
inference time, achieving uniform sampling without the cost of
fine-tuning.

\section{Preliminaries and Notation}
\label{sec:prelim}

In this work, we adopt the manifold
assumption~\citep{song2019generative,de2022convergence,loaizadeep}
as follows:
\begin{assumption}[The Manifold Hypothesis] \label{assumption:manifold}
  We assume that the data distribution $\mudata$ is supported on a compact,
  boundaryless $C^4$ embedded submanifold $\cM \subset \bR^d$, with $\dim(\cM) = n$.
\end{assumption}

\textbf{Local coordinates and manifold geometry.}
Under the manifold hypothesis, the $n$-dimensional manifold $\mathcal{M}$ can be
described locally using coordinates from a flat, Euclidean space. This is done
via a set of smooth mappings, or charts, $\Phi: U \to \mathcal{M}$,
where each chart maps an open set of parameters $U \subset \mathbb{R}^n$ to a
patch on the manifold. For notational simplicity, we will work with a single
chart, where $u \in U$ represents the local coordinates of a point
$\Phi(u)$ on $\mathcal{M}$.
The manifold's intrinsic, and generally non-Euclidean, geometry is captured by
the Riemannian metric tensor, $g(u)$. This tensor
provides the means to measure lengths and angles on the curved surface.
The metric gives rise to the Riemannian volume measure, $d\mathcal{M}(x)$, which is the
natural way to integrate a function $f: \mathcal{M} \to \mathbb{R}$ over the
manifold. In local coordinates, this integral is expressed as
$
\int_{\mathcal{M}} f(x)\, d\mathcal{M}(x) = \int_{U} f(\Phi(u)) \sqrt{\det(g(u))} \, du
$, w.r.t. the Lebesgue measure on $U$.
Here, the term $\sqrt{\det(g(u))}$ is the volume correction factor.
While we use
a single chart for clarity, integration over the entire compact manifold is
handled by stitching together multiple charts via a partition of unity.
The set of points in $\mathbb{R}^d$ that are
sufficiently close to the manifold forms the tubular neighborhood:
$ T_{\mathcal{M}}(\epsilon) \coloneqq \{x \in \mathbb{R}^d : \operatorname{dist}(x, \mathcal{M}) < \epsilon\} $.
For any point $x$ within this neighborhood, there exists a unique closest point 
on the manifold, given by the $P_{\mathcal{M}}(x): T_{\mathcal{M}}(\epsilon) \to \mathcal{M}$.
This projection allows us to define the squared distance function to the manifold, a quantity of central importance to our analysis:
\begin{equation}
    d_\mathcal{M}(x) \coloneqq \frac{1}{2}\mathrm{dist}^2(x, \mathcal{M}) 
    = \min_{\bar x \in \mathcal{M}} \frac{1}{2}\|x - \bar x\|^2.
\end{equation}
Further details and notations regarding the manifold hypothesis are provided in \Cref{sec:apx_prelim}.

\subsection{The Gaussian Smoothed Measure and Connection to Diffusion Models}
\label{subsec:gaussian_smoothed_measure_diffusion_model}

With \Cref{assumption:manifold}, we define the corresponding density $\pdata$
of $\mudata$ with respect to the Lebesgue measure on $U$:
$\pdata(u) \coloneqq \frac{d \left(\Phi^* \mudata\right)}{du} (u)$,
where $\Phi^* \mudata (S) \coloneqq \mudata (\Phi(S))$
for $S \subseteq U$,
and assume the following regularity assumption:
\begin{assumption}[Regularity and Coverage of $\pdata$] \label{assumption:pdata}
    The probability density $\pdata: U \to \bR$ defined
      w.r.t. the Lebesgue measure on $U$ is $C^1(U)$ and strictly positive.
\end{assumption}
Recall the two Gaussian–smoothed measures $\mu_\sigma$ introduced in \Cref{eqn_p_sigma_definition}.
We follow the naming convention of \citet{songscore} and denote by
$\mu_\sigma^{\mathrm{VE}}$ the variance–exploding~(VE) smoothing and by
$\mu_\sigma^{\mathrm{VP}}$ the variance–preserving~(VP) smoothing.
Their densities w.r.t. the Lebesgue measure on $\mathbb{R}^d$ are
\begin{equation}\label{eq:p_sigma}
\begin{aligned}
p_{\sigma}(x)
\coloneqq \int_{\cM}
\frac{1}{(2\pi\sigma^{2})^{d/2}}
\exp\!\left(-\frac{\|x- \gamma(\sigma) \Phi(u)\|^{2}}{2\sigma^{2}}\right)\,
\pdata(u)\,\mathrm{d}u, 
\end{aligned}
\end{equation}
where the densities are denoted
$p_\sigma^{\mathrm{VE}}$ for VE with $\gamma(\sigma)=1$ and
$p_\sigma^{\mathrm{VP}}$ for VP with $\gamma(\sigma)=\sqrt{1-\sigma^2}$.
We take $\pdata$ to be the true population density rather than a finite-sample empirical approximation.

These smoothed distributions correspond to the marginals of the forward
noising processes used in diffusion and score-based generative modeling.
In SMLD or VE-SDE~\citep{songscore}, Gaussian noise with variance
$\sigma^2(t): \mathbb{R}_+ \!\to \mathbb{R}_+$ is added to the data at time $t$,
a model is trained to progressively denoise, and in the reverse process the
objective is to sample from $p_{\sigma(t)}^{\mathrm{VE}}$, recovering $\pdata$
as $t \to 0$ (equivalently, $\sigma(t) \to 0$). Similarly, DDPM or
VP-SDE~\citep{ho2020denoising,songscore} corresponds to the VP density
$p_{\sigma(t)}^{\mathrm{VP}}$, again with the goal of recovering $\pdata$ in the
limit $t \to 0$.  
Beyond the reverse process, one may also directly use the learned score to run a
Langevin sampler targeting $p_\sigma^{\mathrm{VE}}$~\citep{song2019generative} or
$p_\sigma^{\mathrm{VP}}$, or combine Langevin sampling with the reverse process,
as in the Predictor--Corrector algorithm~\citep{songscore}.
Since our results apply to both VE and VP settings, we adopt the unified notation $p_\sigma$
whenever no ambiguity arises.

\subsection{Bayesian Inverse Problems}
Another important algorithmic implication of our results concerns Plug-and-Play
(PnP) methods for Bayesian inverse problems \citep{venkatakrishnan2013plug}. 
Let
$x\in\mathbb R^d$ be the latent signal and $y\in\mathcal Y\subseteq\mathbb R^m$
the observation $y \;=\; A(x) + \xi$, where $A:\mathbb R^d\to\mathbb R^m$ is the measurement map and $\xi\in\mathbb R^m$ is noise. Under standard assumptions on $A$ and $\xi$ (e.g., $A$ linear, $\xi\sim\mathcal N(0,s^2 I)$), the likelihood admits a density $p(y\mid x)\propto \exp\!\big(-v(x;y)\big)$ (for the Gaussian case, $v(x;y)=\tfrac{1}{2s^2}\|A(x)-y\|^2$). In the Bayesian framework we endow $x$ with a prior $\pprior$. 
Inference is cast as sampling from the posterior
$p(x\mid y)\;=\; p(y\mid x)\, \pprior(x) / \int p(y\mid \bar x)\, \pprior(\bar x) d\bar x$.

\paragraph{Plug-and-Play (PnP).}
PnP methods address the case where the prior is (i) known up to a normalizing constant, e.g. a Gibbs measure or (ii) only accessible via samples (common in ML). 
A unifying sampling paradigm is posterior Langevin with a (possibly learned) prior score $\hat s \simeq \nabla\log \pprior$,
\begin{equation} \label{eqn_PnP}
    \mathrm dX_t \;=\; -\nabla_x v(X_t; y)\,\mathrm dt \;+\; \hat s(X_t)\,\mathrm dt \;+\; \sqrt{2}\,\mathrm dW_t.
\end{equation}
In case (ii), $\hat s$ is a score estimator obtained, e.g., by score matching on prior samples. 
A common choice of $\pprior$ would be the density $p_\sigma$ defined in \cref{eq:p_sigma} with small $\sigma$. In this context, to ensure update (\ref{eqn_PnP}) yields samples matching the target posterior distribution, existing works require the learned score $\hat s$ to be at least $o(1)$ accurate \citep{laumont2022bayesian}, or even exact \citep{pesme2025map}.

\subsection{Stationary Distribution for Non-reversible Dynamics}
\label{subsec:pre_wkb}

In score learning, one typically learns a score function $s(x, \epsilon)$ for a target
density and then runs Langevin dynamics (equivalently, the corrector step in
the Predictor--Corrector algorithm for diffusion models~\citep{songscore}) until near stationarity to sample from that density:
\begin{align*}
  dX_t \;=\; s(X_t, \epsilon)\,dt \;+\; \sqrt{2}\, dW_t .
\end{align*}
If $s(x, \epsilon) = -\nabla f_{\epsilon}(x)$, the stationary distribution is
proportional to $\exp(-f_{\epsilon}(x))$. In practice, however, the score is often
produced by a parameterized model and need not be a gradient field (this is also
the case for our proposed algorithms). The resulting Langevin dynamics is then
generally \emph{non-reversible}, and its stationary distribution need not admit
a closed form—an open problem studied in, e.g.,
\citep{graham1984weak,maes2009nonequilibrium,rey2015irreversible}.

Several works have sought to characterize the stationary distribution of non-reversible SDEs. Notably,
\citet{matkowsky1977exit,maier1997limiting,graham1984weak,bouchet2016generalisation}
employ the WKB
ansatz~\citep{wentzel1926verallgemeinerung,kramers1926wellenmechanik,brillouin1926mecanique},
which is commonly used in matched asymptotic expansions~\citep{holmes2012introduction}.
This approach posits that the stationary density takes the form
\begin{align} \label{eq:wkb_expansion}
  \exp\left(-\tfrac{V(x)}{\epsilon}\right)\, c_{\epsilon}(x),
  \quad \text{with} \quad c_{\epsilon}(x) \;=\; \sum_{i=0}^{k} c_i(x)\,\epsilon^{i},
\end{align}
for some $k\in\mathbb{N}$. The functions $V$ and $\{c_i\}$ are then identified
by inserting~\eqref{eq:wkb_expansion} into the stationary Fokker--Planck
equation and balancing terms order by order in $\epsilon$. Importantly, prior analyses typically focus on low-dimensional special examples or on drifts with a \emph{single} stable point. The difficulty of removing such restrictions turn out to be central to our analysis; see \Cref{sec:recover_uniform} for details.

\section{Central Insight: Gaussian Smoothing Recovers Geometry Before Distribution}
\label{sec:Gaussian-smoothing}

This section presents the central insight of the paper: While the proofs of our later main results are technically involved, they are all guided by a common intuition that is transparent and can be understood through a simple Taylor expansion of $\log p_{\sigma}$ at $\sigma = 0$:
\begin{theorem}[{Informal \Cref{lemma:limit_V}}] \label{thm_main_informal}
Assume \Cref{assumption:manifold,assumption:pdata} holds. For any $x \in T_{\cM}(\epsilon)$,
\begin{equation} \label{eq:score_expansion}
  \begin{aligned}
    \log p_{\sigma}(x)
    = -\frac{1}{\sigma^2} d_\cM(x) + \log \pdata(\Phi^{-1}(\Pm(x)))
      - \tfrac{d-n}{2} \log(2\pi \sigma^2) 
      + H(x)
       + o(1),
  \end{aligned}
\end{equation}
where $H(x)$ contains the curvature information of the manifold and $\epsilon$
is some sufficiently small constant; both of them are independent of $\sigma$.
The small $o(1)$ term is uniform for $x \in T_{\cM}(\epsilon)$.
\end{theorem}

From \Cref{eq:score_expansion}, it follows immediately that the scaled log-density recovers the distance function to the manifold in the 
small $\sigma$ limit:
\begin{equation}
    \lim_{\sigma \to 0} \, \sigma^2 \log p_\sigma(x) \;=\; - d_\cM(x) 
    \quad \text{uniformly for all } x \in T_{\cM}(\epsilon).
\end{equation}

The appearance of $d_\cM(x)$ under the manifold hypothesis should not come as a surprise; indeed, as $p_\sigma \to \pdata$ when $\sigma \to 0$, and since $\pdata$ is supported entirely on $\cM$, any point $x$ with $d_\cM(x) > 0$ must be assigned zero probability, which explains the divergent scaling factor $\sigma^{-2}$ in the coefficient. What is more surprising is that \emph{only} $d_\cM(x)$ appears at leading order, with \emph{no dependence on $\pdata$}: Information about $\pdata$ enters only at the higher-order terms of $\Theta(1)$.

This reveals a fundamental \emph{rate separation}: for \emph{any} distribution
supported on $\cM$, one must first recover $d_\cM(x)$ \emph{exactly} before
learning anything about $\pdata$, as any inaccuracy in $d_\cM(x)$ gets blown up
by the diverging factor $\sigma^{-2}$. Moreover, coefficients encoding $\pdata$
appear at order $O({\sigma^{-2}})$ higher, meaning that extracting information
about $\pdata$ requires a level of accuracy orders of magnitude stricter than
that needed to recover the manifold geometry, i.e., the distance function
$d_\cM$.

As demonstrated in \Crefrange{sec:scale}{sec:recover_uniform}, this observation entails several significant consequences for machine learning. Each of these can be understood as a manifestation of the fundamental rate separation between geometric recovery \vs distributional learning established in \Cref{thm_main_informal}.

\section{Scale Separation in Existing Generative Learning: Geometry versus Distribution}
\label{sec:scale}

In this section, we study the paradigm of existing generative learning where algorithms
target to learn the Gaussian-smoothed measure $\mu_\sigma$, such as the diffusion
models discussed in \Cref{subsec:gaussian_smoothed_measure_diffusion_model}.
We denote the corresponding perfect score function by $s^*(x,\sigma) \coloneqq \nabla \log p_\sigma(x).$

In practice, however, the generated samples may follow a different distribution
due to imperfections such as errors in training or discretization of the
reverse differential equation. We therefore let
$\pi_{\sigma}(x): \bR^d \to \bR$
denote the density of the distribution actually
produced by an empirical algorithm, and define its associated score as $s_{\pi_{\sigma}}(x) \coloneqq \nabla \log \pi_{\sigma}(x).$
Our analysis focuses on $\pi_{\sigma}$ in terms of discrepancies between
$s_{\pi_{\sigma}}(x)$ and the ideal score $s^*(x,\sigma)$. 

Before presenting our result, we impose the following assumption on the recovered distribution.
\begin{assumption} \label{assumption:learned_score}
  We denote the log-density of the recovered distribution as
  $-f_{\sigma} \coloneqq \log \pi_{\sigma}(x)$, and assume that
    $f_{\sigma}$ is $C^1(K)$.
   Furthermore, we impose the following conditions:
  \begin{enumerate}
    \item There exists a compact set $K \subset \bR^d$ with
      $T_{\cM}(\epsilon) \subset K$ such that the density concentrates on $K$ as
      $\sigma \to 0$, i.e., $ \lim_{\sigma \to 0} \int_{K} \pi_{\sigma}(x)\, dx = 1$.
    \item $K$ is uniformly rectifiably path-connected, meaning that for any
      two points $x, y \in K$, there exists a path in $K$ connecting $x$ and $y$ whose
      length is uniformly bounded for all $x,y \in K$.
  \end{enumerate}
  \end{assumption}
  
\begin{remark}
  We believe our assumptions are already reflected in practice:
Since $\pi_\sigma$ represents the effective distribution of the generated samples, it can incorporate standard constraints such as data clipping (e.g., to $[-1, 1]$) used in many diffusion models~\citep{ho2020denoising,saharia2022photorealistic}. This ensures the generated density concentrates on a compact set $K$ as required. Furthermore, such regular sets are naturally uniformly rectifiably path-connected.
\end{remark}

We are ready to state our main result in this section; see \Cref{subsec:proof_scale} for the proof.
\begin{theorem} \label{thm:recover_pdata}
  Suppose \Cref{assumption:manifold,assumption:pdata,assumption:learned_score} hold.
  Denote the score error as
  \begin{align*}
    E_{\sigma} \coloneqq \norm{ s_{\pi_{\sigma}} - s^*(\cdot, \sigma) }_{L^{\infty}(K)}.
  \end{align*}
\begin{enumerate}
  \item \textbf{Concentration on Manifold.} If we have that
    $\bm{E_{\sigma} = o(\sigma^{-2})}$,
  then $\pi_{\sigma}$ concentrates on $\cM$, i.e.,
  \begin{align*}
    \lim_{\sigma \to 0} \int_{\dist(x, \cM) > \delta} \pi_{\sigma}(x) dx = 0
    \quad \text{for any} \quad \delta > 0.
  \end{align*}

    \item \textbf{Arbitrary Distribution Recovery.} For any distribution $\hat \pi$ supported on $\cM$ with
  $C^1$ density, one can construct $f_{\sigma}$ such that $\bm{E_{\sigma} = \Omega(1)}$
  as $\sigma \to 0$,
  and $\pi_{\sigma}$ converges weakly to $\hat \pi$.
    
  \item \textbf{Recovering $\pdata$.} If we have that $\bm{E_{\sigma} = o(1)}$
    as $\sigma \to 0$,
  then $\pi_{\sigma}$ converges weakly to $\pdata$. 
\end{enumerate}
\end{theorem}

This result formalizes the intuitive fact that recovering $\pdata$ requires $\nabla\log \pi_\sigma$ to match the true score to within $o(1)$ accuracy as $\sigma \to 0$.
The reason is clear from the expansion \eqref{eq:score_expansion}: the
distribution $\pdata$ only appears in the $\Theta(1)$ term, and any larger error
would overwhelm this information. In practice, however, achieving such accuracy
is extremely challenging, particularly in the small-$\sigma$ regime.
However, recovering the manifold is simple---only $o(1/\sigma^2)$ accuracy is
required such that as $\sigma \to 0$, the density will concentrate on $\cM$---a shape separation from recovering $\pdata$.

\textbf{Implications to Diffusion Models.}
\newcommand{\pest}{\hat{p}}
\newcommand{\err}{\mathrm{err}}
As we mentioned before, the paradigmatic example to which our results can be
applied is diffusion models.
Our Theorem~\ref{thm:recover_pdata} then reveals a sharp scale separation in
terms of the score error:
\emph{well before the true distribution $\pdata$ is
fully recovered, one can already recover a distribution supported on the same
data manifold}. In practice, this often suffices, as what truly matters is
capturing the \emph{structural features} of the manifold—realistic images,
plausible protein conformations, or meaningful material geometries. This insight
provides a potential new explanation for the remarkable success of diffusion
models.

\section{New Paradigm of Geometric Learning: Recover Uniform Distributions with
\texorpdfstring{$o(\sigma^{-2})$}{o(σ⁻²)} Score Error}
\label{sec:recover_uniform}

As shown in \Cref{thm:recover_pdata}, while concentration on the manifold is orders of magnitude simpler, the recovered distribution can still be \textbf{arbitrary} unless the score is learned with $o(1)$ accuracy. In contrast, we show in this section the striking fact that even with score errors as large as $o(\sigma^{-2})$, with a simple modification of the existing algorithm, one can recover the \emph{uniform distribution on the manifold}—a fundamental distribution that plays a key role in scientific discovery and encodes rich geometric information about the manifold \citep{de2025provable, belkin2008towards}.

Unlike in \Cref{sec:scale}, where we compared errors by evaluating a learned \emph{distribution} $\pi_\sigma$ against the ideal $p_\sigma$ through their score functions, in this section we assume direct access to an estimated \emph{score oracle} $s(\cdot, \sigma)$, such as those learned via score matching in diffusion models. Given access to such an oracle, our proposed algorithm consists of running the following SDE for some $\alpha > 0$:
\begin{align} \label{eq:modified_langevin}
  dX_t = \sigma^{\alpha} s(X_t, \sigma)\, dt + \sqrt{2}\, dW_t,
\end{align}
which we refer to as the \emph{Tempered Score} (TS) Langevin dynamics. We claim that, under mild error assumptions, the stationary distribution of this SDE, denoted $\tilde{\pi}_{\sigma}$, converges to the uniform distribution on the manifold as $\sigma \to 0$. 

Our analysis proceeds in two steps. First, we establish the result in a simplified setting where the score oracle $s(\cdot,\sigma)$ is guaranteed to be a gradient field, with a proof analogous to \Cref{sec:scale}. Second, we tackle the substantially more challenging case in which no \emph{a priori} gradient structure is assumed. Full proofs are provided in \Cref{subsec:proof_recover_uniform}.

\textbf{Warm-up: Score Oracle is a Gradient Field.}
We use the same notation as in \Cref{sec:scale}, namely
$s(x,\sigma) = -\nabla f_{\sigma}(x)$. In this case, the stationary distribution
of \Cref{eq:modified_langevin} admits the explicit form
\begin{align*}
  \tilde \pi_{\sigma}(x) \propto \exp\!\left(- \sigma^{\alpha} f_{\sigma}(x)\right).
\end{align*}
We then obtain the following result, using a proof technique similar to that of
\Cref{thm:recover_pdata}.
\begin{theorem} \label{thm:recover_uniform_gradient}
  Assume \Cref{assumption:manifold,assumption:pdata,assumption:learned_score}
  hold, with $\pi_{\sigma}$ replaced by $\tilde \pi_{\sigma}$.
  Suppose
  \begin{align} \label{eq:score_error_uniform}
    \norm{ s(\cdot, \sigma) - s^*(\cdot, \sigma) }_{L^{\infty}(K)}
      = o\!\left(\sigma^{\beta}\right)
      \quad \text{for some } \beta > -2.
  \end{align}
  Then for any $\max\{-\beta,0\} < \alpha < 2$, as $\sigma \to 0$,
  $\tilde \pi_{\sigma}$ converges weakly to the
  \textbf{uniform distribution} on the manifold $\cM$ with respect to the
  intrinsic volume measure. More precisely, the limiting distribution
  $\tilde \pi$ with respect to the Lebesgue measure on $U$ satisfies
  \[
    \tilde \pi(u) \propto \frac{d\cM}{du}(u),
  \]
  where $(d\cM / du)(u) = \sqrt{\det(g(u))}$ is the Riemannian volume element on
  $\cM$.
\end{theorem}

\textbf{General Non-Gradient Score Oracle.}
While \cref{thm:recover_uniform_gradient} already illustrates the rate separation phenomenon we wish to emphasize, it relies on the highly impractical assumption that the estimated scores $s(\cdot,\sigma)$ are exact gradient fields. To enhance the applicability of our framework, it is crucial to relax this stringent assumption.

As discussed in \Cref{subsec:pre_wkb}, existing approaches to non-gradient scores (and hence non-reversible dynamics) typically assume the existence of a unique point $x^*$ such that $\lim_{\sigma \to 0}\sigma^\alpha s(x^*,\sigma) = 0$, with the key consequence of collapsing the prefactor $c_0$ in \eqref{eq:wkb_expansion} to a normalization constant $c_0(x^*)$. Our framework, however, explicitly violates this assumption: we require that $\lim_{\sigma \to 0}\sigma^\alpha s(\cdot,\sigma)$ stabilizes to a \emph{manifold} rather than a singleton. Under this setting, the limiting behavior of $c_0$ is far from obvious, and the resolution of this issue turns out to be highly nontrivial. 

To this end, a central part of our proof is devoted to showing that $c_0$ nevertheless remains constant, albeit for an entirely different reason: we prove that the higher-order terms in the Fokker--Planck expansion enforce $c_0$ to satisfy a \emph{parabolic PDE} on the manifold, and by the strong maximum principle~\citep{gilbarg1977elliptic}, the only solutions on a compact manifold are constants.

With these techniques, we obtain the same conclusion as
\Cref{thm:recover_uniform_gradient}:
\begin{theorem} \label{thm:recover_uniform}
Assume \Cref{assumption:manifold,assumption:pdata,eq:score_error_uniform}
hold, and further suppose $\pdata \in C^2(U)$. For any
$\max\{-\beta,0\} < \alpha < 2$, assume that the SDE admits a unique
stationary distribution, denoted $\tilde \pi_{\sigma}$, which locally admits a
WKB form (\Cref{assumption:wkb} with $\theta = \sigma^{2-\alpha}$). Then the
conclusion of \Cref{thm:recover_uniform_gradient} holds.
\end{theorem}

Setting $\alpha = 0$ in \cref{eq:modified_langevin} recovers the standard Langevin sampler or the ``Corrector'' step commonly used in diffusion-based sampling~\citep{songscore}. Our results in \Cref{thm:recover_uniform_gradient,thm:recover_uniform} therefore imply that a simple, one-line modification of these standard schemes is enough to recover the uniform distribution on the data manifold \emph{from samples of $\pdata$}, even when the score error is as large as $o(\sigma^{-2})$—a substantially weaker requirement than the $o(1)$ accuracy needed to recover $\pdata$ itself.

\begin{remark}
In \Cref{sec:ts_convergence}, we provide further discussion on the convergence (mixing time) of TS Langevin compared to standard Langevin dynamics. While characterizing the general convergence rate is a non-trivial problem left for future work, our analysis indicates that TS Langevin maintains comparable algorithmic efficiency. In fact, by analyzing the Poincaré constant, we identify concrete examples where TS Langevin converges provably exponentially faster than standard, untempered Langevin dynamics.
\end{remark}

\section{Uniform Prior is More Robust For Bayesian Inverse Problems}
\label{sec:bayesian_inverse}

In Bayesian learning, one often sets the prior $\pprior$ to the Gaussian-smoothed data distribution $p_{\sigma}$ defined in \Cref{eq:p_sigma} with some small smoothing parameter $\sigma$. To ensure asymptotically correct posterior samples under this choice, the learned score typically must be exact \citep{pesme2025map}, $\hat s = \nabla \log p_\sigma$, or achieve vanishing error, $\|\hat s - \nabla \log p_\sigma\|_{\mathcal L^\infty}=o(1)$ \citep[Proposition 3.3 and H2]{laumont2022bayesian}. In contrast, under our framework, if one adopts the manifold volume measure (i.e., the uniform distribution on $\mathcal M$) as the prior, then correct posterior sampling can be attained under a substantially weaker requirement: it suffices that the score error scales as $o(\sigma^{-2})$. The precise statement is given in the theorem below.%

\begin{theorem} \label{thm:with_guidence}
  Under the same assumptions as in \Cref{thm:recover_uniform}, and suppose
  $v: \bR^d \to \bR$ is bounded on $\bR^d$,
  and $C^1$ on $T_{\cM}(\epsilon)$.  
  Then, as $\sigma \to 0$, the stationary distribution of the SDE
  \begin{align} \label{eq:modified_langevin_with_guidence}
      dx_t = -\nabla v(x_t)\, dt - \sigma^{\alpha} \nabla f_{\sigma}(x_t)\, dt
      + \sqrt{2}\, dW_t,
  \end{align}
  converges weakly to a distribution supported on $\cM$ with density 
  $\propto \exp\bigl(-v(\Phi(u))\bigr)\, \frac{d\cM}{du}(u)$.
\end{theorem}

\textbf{Diffusion Models with Classifier-Free Guidance.}
The above result can also be applied to diffusion models.
The drift term in \Cref{eq:modified_langevin_with_guidence} represents the
effective score of a diffusion model with classifier-free guidance~\citep{ho2022classifier}. In this formulation, $-\nabla f_{\sigma}$
denotes the unconditional score estimate, while the guidance term
$-\nabla v$ equals the guidance scale $w$ times the difference between the
conditional and unconditional score estimates. 
Our tempered score can be applied directly to CFG diffusion models with a
Predictor--Corrector sampler: in the corrector (Langevin) step, replace the
score by its tempered version according to
\Cref{eq:modified_langevin_with_guidence} (i.e., scale the unconditional score by
$\sigma^{\alpha}$). We will demonstrate the effectiveness of this modification
empirically in \Cref{sec:exp_diffusion}.

\vspace{-0.5em}
\section{Experiments}
\label{sec:exp}
\vspace{-0.5em}

To empirically validate our theory, we present preliminary experiments on both
simple synthetic manifolds and a real-world image–generation setting with
diffusion models. On synthetic manifolds, we directly verify the claims of
\Cref{sec:recover_uniform}, demonstrating recovery of the uniform
distribution on the manifold. In the image domain, we show that our proposed
algorithm yields samples that are both more diverse and high-quality. Further
experimental details are provided in \Cref{sec:exp_details}.

\subsection{Numerical Simulations on Ellipse}
\vspace{-0.5em}

In this subsection, we illustrate our theoretical results with numerical simulations.
We consider a simple manifold given by an ellipse embedded in the two-dimensional
Euclidean space, $\cM = \{ (x,y) \in \bR^2 \mid \left(x/a\right)^2 +
  \left(y/b\right)^2 = 1 \}, \quad a,b > 0$,
   and $\pdata$ is chosen to be a von Mises distribution supported on the angular
parameterization of the ellipse. The score function is parameterized using a
transformer-based neural network, trained with the loss function introduced in
\citep{song2019generative}. After training, we evaluate the learned score
function with $\sigma = 10^{-2}$ and perform Langevin dynamics until convergence.
Training hyperparameters are tuned to minimize the test loss.

\begin{figure}[t]
\centering
\begin{subfigure}[t]{0.22\textwidth}
  \centering
  \includegraphics[width=\textwidth]{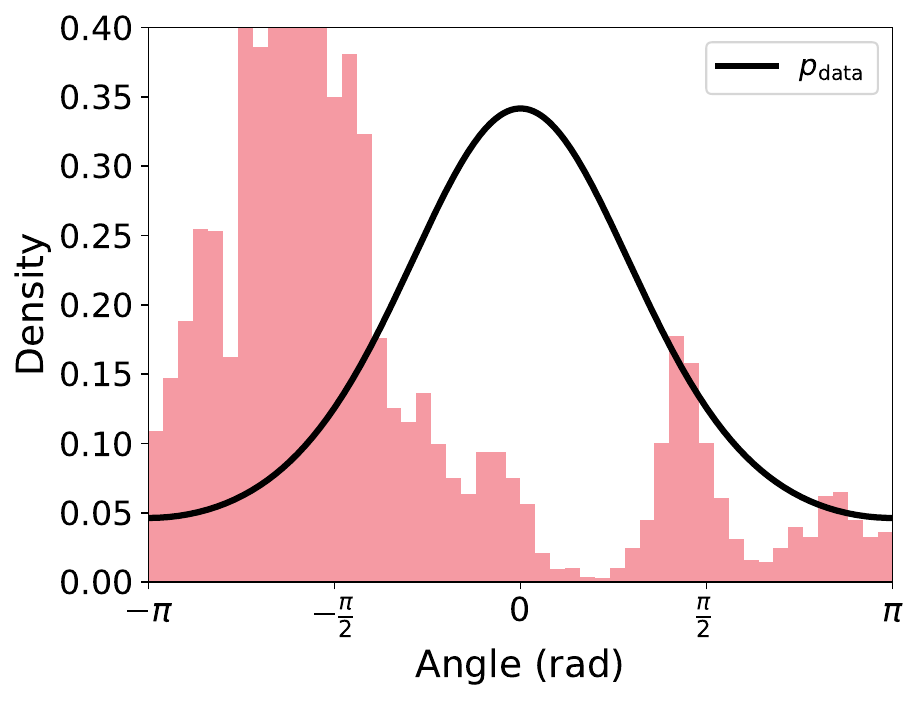}
  \caption{L (ellipse)}
  \label{fig:b_2_alpha_0}
\end{subfigure}
\hfill
\begin{subfigure}[t]{0.22\textwidth}
  \centering
  \includegraphics[width=\textwidth]{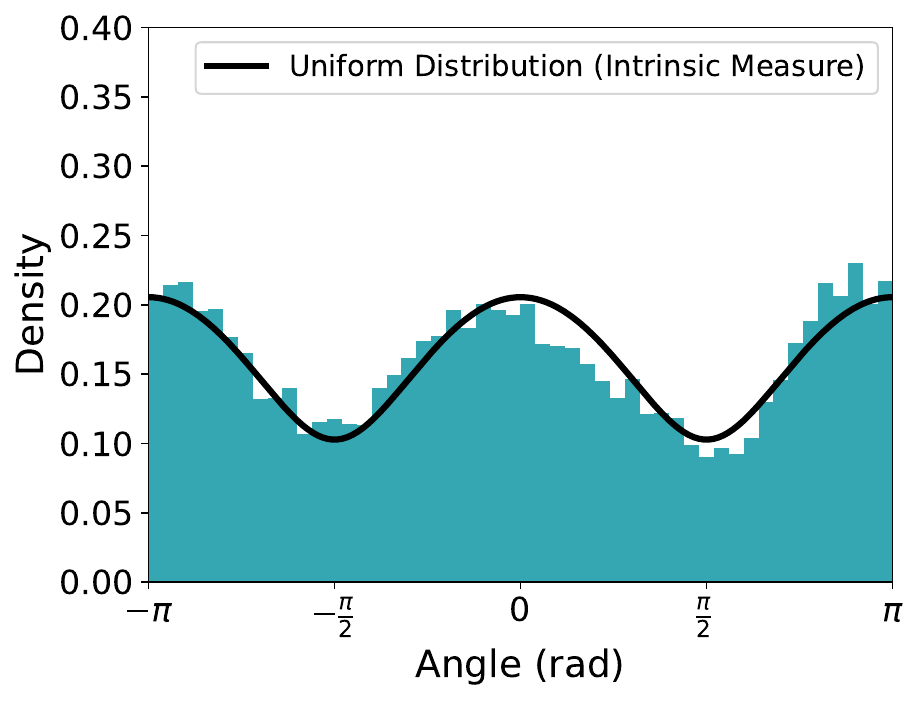}
  \caption{TS-1 (ellipse)}
  \label{fig:b_2_alpha_1}
\end{subfigure}
\hfill
\begin{subfigure}[t]{0.22\textwidth}
  \centering
  \includegraphics[width=\textwidth]{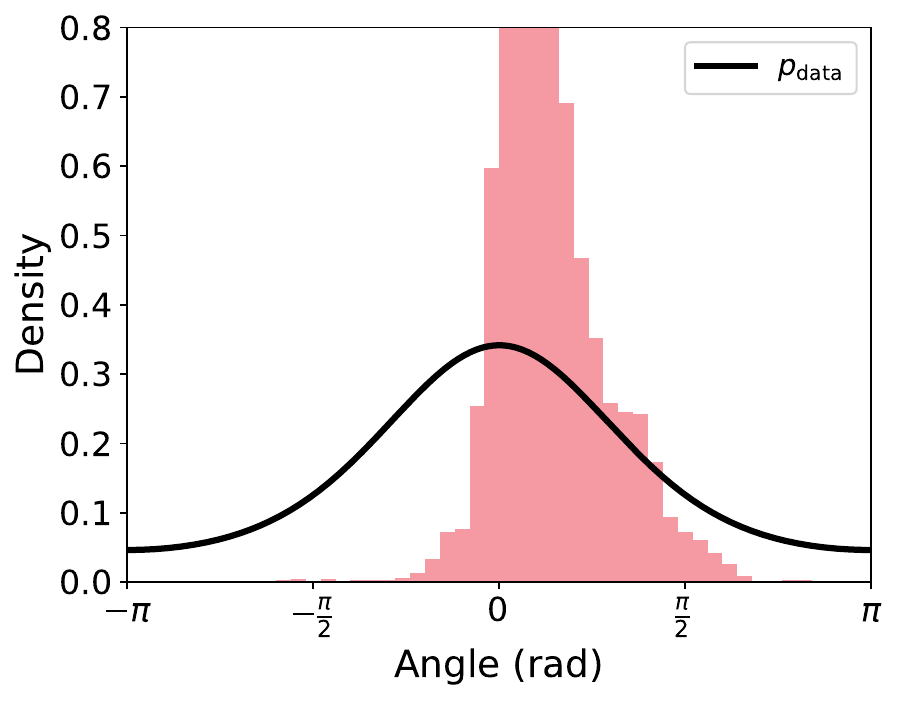}
  \caption{L (circle)}
  \label{fig:b_1_alpha_0}
\end{subfigure}
\hfill
\begin{subfigure}[t]{0.22\textwidth}
  \centering
  \includegraphics[width=\textwidth]{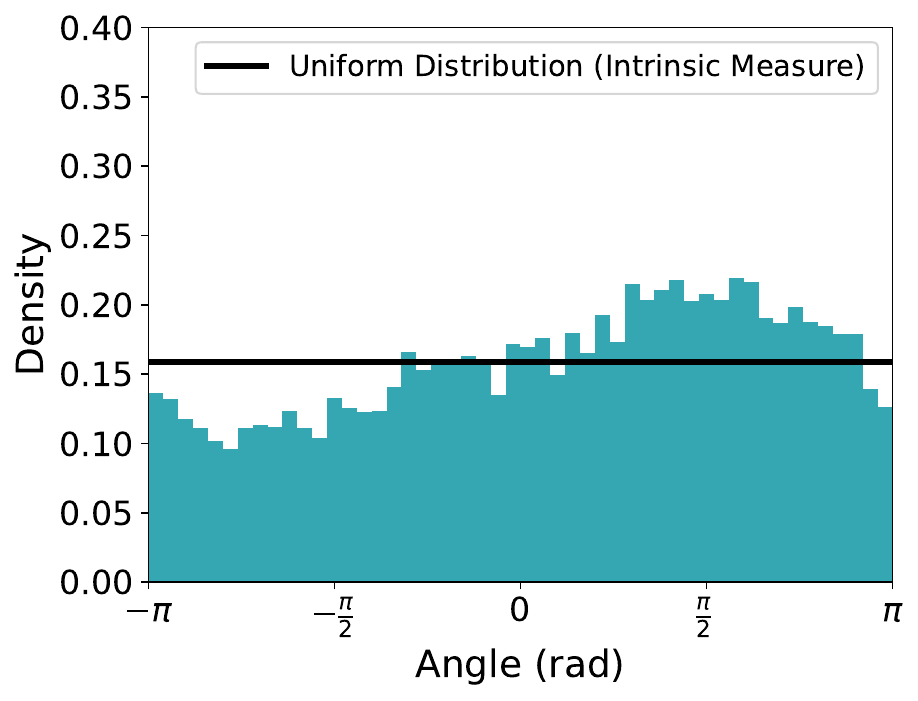}
  \caption{TS-1 (circle)}
  \label{fig:b_1_alpha_1}
\end{subfigure}
\caption{Comparison of stationary sample distributions generated with standard Langevin dynamics (L)
versus our Tempered Score Langevin dynamics~\Cref{eq:modified_langevin} with $\alpha = 1$ (TS-1).
The circle and ellipse correspond to manifolds with $(a,b) = (1,1)$ and $(a,b) = (1,2)$,
respectively.}
\label{fig:ellipse}
\end{figure}

As shown in \Cref{fig:ellipse}, the stationary distribution produced by standard
Langevin dynamics deviates substantially from $\pdata$, even in this simple
elliptical setting, highlighting the difficulty of accurately learning the score
function at small $\sigma$. In contrast, our TS Langevin dynamics reliably
recovers the uniform distribution on the manifold, in agreement with
\Cref{thm:recover_uniform}.

\vspace{-0.5em}
\subsection{Image Generation with Diffusion Models}
\label{sec:exp_diffusion}
\vspace{-0.5em}

\begin{table}[t]
  \centering
  \setlength{\tabcolsep}{4pt}
  \begin{tabular}{lccc ccc cc}
  \toprule
  \textbf{Prompt}
  & \multicolumn{2}{c}{\textbf{Furniture}}
  & \multicolumn{2}{c}{\textbf{Car}}
  & \multicolumn{2}{c}{\textbf{Architecture}} \\
  \cmidrule(lr){2-3} \cmidrule(lr){4-5} \cmidrule(lr){6-7}
  \textbf{Method} & P-sim$\uparrow$ & I-sim$\downarrow$ & P-sim & I-sim & P-sim & I-sim \\
  \midrule
  DDPM          & 29.56 & 80.78 & 26.23 & 87.30 & \textbf{27.36} & 81.53 \\
  PC            & 29.40 & 81.24 & 26.30 & 87.20 & 27.13 & 81.03 \\
  \emph{TS (ours)} & \textbf{30.20} & \textbf{80.76} & \textbf{26.62} & \textbf{87.14} & 27.32 & \textbf{80.76} \\
  \bottomrule
  \end{tabular}
  \caption{Comparison of images generated by DDPM, PC, and TS. The prompts used are
  ``Creative furniture,’’ ``An innovative car design,’’ and ``A creative
  architecture.’’ For PC and TS, the number of corrector steps and
  $\alpha$ (for TS) are tuned.}
  \label{tab:best_result}
  \vspace{-0.5em}
  \end{table}

To validate our theoretical findings in a practical, large-scale setting, we
conducted experiments on image generation. We demonstrate that a one-line
modification to the widely-used Predictor-Corrector (PC) sampling algorithm
\citep{songscore} can enhance both the quality and diversity of images generated
by a pre-trained diffusion model. These experiments serve as a proof of concept,
applying our proposed Tempered Score (TS) method to off-the-shelf diffusion 
models.
Our modification targets the corrector step of the PC algorithm, which uses
Langevin dynamics to refine the sample at each stage of the reverse process. In
our TS method, we scale the unconditioned score prediction by a factor of
$\sigma^{\alpha}$, as motivated by our analysis and discussion in \Cref{sec:bayesian_inverse}.
The standard classifier-free guidance term, i.e., $\nabla v$ in \Cref{eq:modified_langevin_with_guidence},
remains unchanged.
Specifically, we compare Stable Diffusion 1.5~\citep{Rombach_2022_CVPR}
with a DDPM sampler~\citep{ho2020denoising},
DDPM with PC sampler, and DDPM with our TS sampler.

We evaluate the performance using two metrics derived from CLIP scores
\citep{hessel2021clipscore}, which measure the cosine similarity between feature
embeddings. \textbf{Quality}: We use the CLIP Prompt Similarity (P-sim), defined
as the average CLIP score between the generated images and their
corresponding text prompt. A higher P-sim value indicates better alignment
with the prompt and thus higher image quality.
\textbf{Diversity}: We use the CLIP Inter-Image Similarity (I-sim),
which is the average pairwise CLIP score between all images generated
with the same prompt. A lower I-sim value means greater diversity among the
samples.

The experimental results in \Cref{tab:best_result} and
\Cref{tab:corrector_steps} provide empirical validation of our theoretical
framework. Our proposed TS method consistently generates more diverse images
than the DDPM and standard PC baselines across three distinct prompts, while
maintaining very high image quality.  
In particular, \Cref{tab:corrector_steps} shows that, for all numbers of
corrector steps considered, TS outperforms standard PC in nearly every case.
Crucially, these improvements are robust to the choice of $\alpha$ and are not merely the result of a larger tuning budget; as demonstrated in \Cref{tab:corrector_steps}, simply setting $\alpha = 1$ without further tuning is sufficient to consistently enhance both quality and diversity compared to the baseline.
Examples of the generated images by PC and TS are shown in
\Cref{fig:generated-samples}.

\begin{table}[t]
  \centering
  \setlength{\tabcolsep}{3pt}
  \resizebox{0.90\textwidth}{!}{
  \begin{tabular}{ll cc cc cc cc cc}
  \toprule
  \multicolumn{2}{c}{\textbf{Num. Corrector Steps}} & \multicolumn{2}{c}{\textbf{5}} & \multicolumn{2}{c}{\textbf{10}} & \multicolumn{2}{c}{\textbf{15}} & \multicolumn{2}{c}{\textbf{20}} & \multicolumn{2}{c}{\textbf{30}} \\
  \cmidrule(lr){3-4} \cmidrule(lr){5-6} \cmidrule(lr){7-8} \cmidrule(lr){9-10} \cmidrule(lr){11-12}
  \textbf{Prompt} & \textbf{Method} & P-sim$\uparrow$ & I-sim$\downarrow$ & P-sim & I-sim & P-sim & I-sim & P-sim & I-sim & P-sim & I-sim \\
  \midrule
  \textbf{Furniture} & PC       & 29.40 & 81.34 & 29.30 & 81.24 & 29.32 & 81.64 & 28.98 & 81.72 & 28.67 & 82.33 \\
                     & \emph{TS (ours)} & \textbf{29.54} & \textbf{81.11} & \textbf{29.58} & \textbf{80.95} & \textbf{29.68} & \textbf{81.34} & \textbf{29.52} & \textbf{81.15} & \textbf{29.43} & \textbf{81.87} \\
  \midrule
  \textbf{Car} & PC           & 26.20 & 87.20 & 26.30 & 87.57 & 26.24 & 87.98 & 26.26 & \textbf{88.06} & 26.17 & 87.94 \\
               & \emph{TS (ours)}  & \textbf{26.23} & \textbf{87.14} & \textbf{26.37} & \textbf{87.42} & \textbf{26.32} & \textbf{87.88} & \textbf{26.28} & 88.07 & \textbf{26.20} & \textbf{87.87} \\
  \midrule
  \textbf{Architect.} & PC    & 27.13 & 81.83 & 27.13 & 81.81 & 26.92 & 81.64 & 26.87 & 81.60 & 26.60 & 81.03 \\
                        & \emph{TS (ours)} & \textbf{27.23} & \textbf{81.58} & \textbf{27.27} & \textbf{81.57} & \textbf{27.14} & \textbf{81.54} & \textbf{27.06} & \textbf{80.97} & \textbf{26.84} & \textbf{80.76} \\
  \bottomrule
  \end{tabular}
  }
  \caption{Comparison of images generated by PC and TS across different numbers of
  corrector steps. For TS, $\alpha = 1$ is used without further tuning. The prompts are the same as in \Cref{tab:best_result}. }
  \label{tab:corrector_steps}
\end{table}

\begin{figure}[t]
  \centering
  \begingroup
  \setlength{\tabcolsep}{0pt}
  \renewcommand{\arraystretch}{0}

  \newcommand{\imgheight}{1.5cm} %
  \newcommand{\imgdir}{imgs/generated_samples}

  \begin{tabular}{*{8}{c}@{}}
    \includegraphics[height=\imgheight]{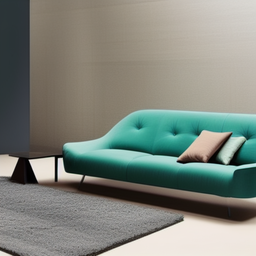} &
    \includegraphics[height=\imgheight]{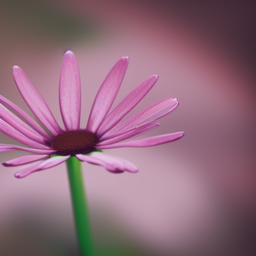} &
    \includegraphics[height=\imgheight]{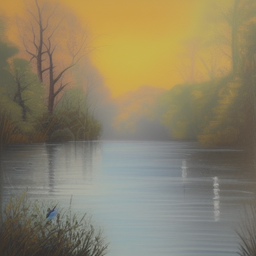} &
    \includegraphics[height=\imgheight]{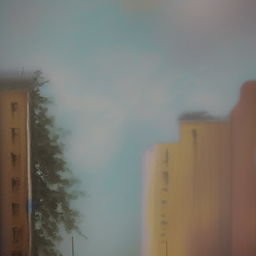} &
    \includegraphics[height=\imgheight]{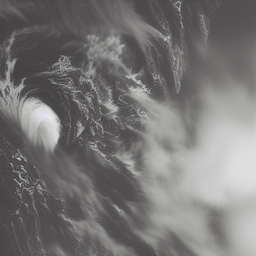} &
    \includegraphics[height=\imgheight]{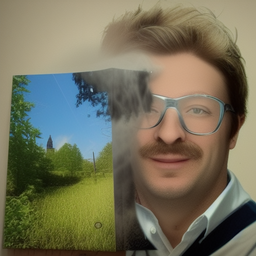} &
    \includegraphics[height=\imgheight]{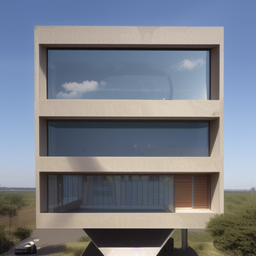} &
    \includegraphics[height=\imgheight]{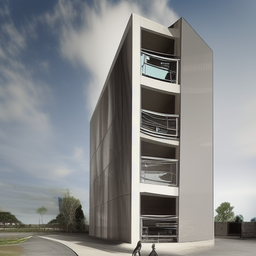}
    \\
    \includegraphics[height=\imgheight]{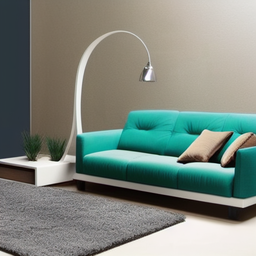} &
    \includegraphics[height=\imgheight]{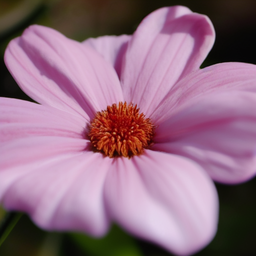} &
    \includegraphics[height=\imgheight]{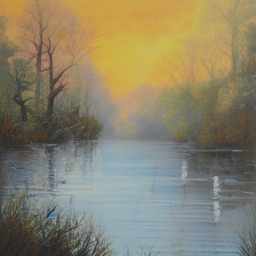} &
    \includegraphics[height=\imgheight]{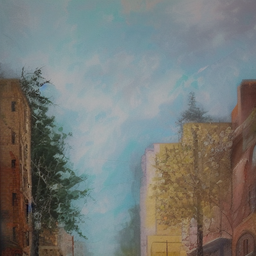} &
    \includegraphics[height=\imgheight]{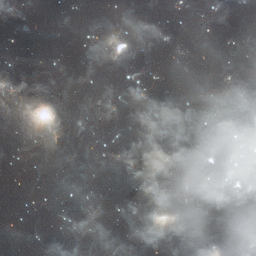} &
    \includegraphics[height=\imgheight]{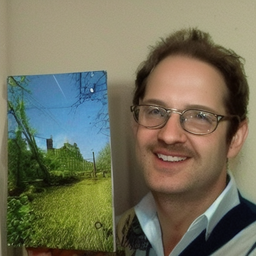} &
    \includegraphics[height=\imgheight]{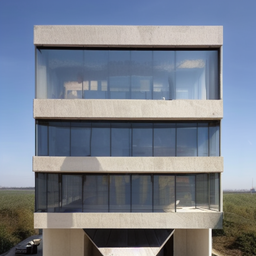} &
    \includegraphics[height=\imgheight]{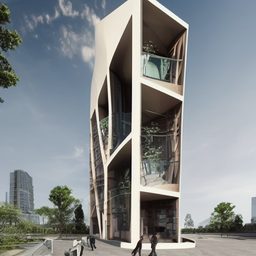}
  \end{tabular}
  \endgroup
  \caption{Top row: PC. Bottom row: \emph{TS (ours)}. Samples in the same column are
  generated using the same prompt, the same number of corrector steps, and the
  same random seed. As shown, TS produces samples that appear more authentic and
  contain richer details. }
  \label{fig:generated-samples}
  \vspace{-0.5em}
\end{figure}

\vspace{-0.5em}
\section{Conclusion}
\vspace{-0.4em}
This paper advocates for a paradigm shift in score-based learning, moving from
the difficult goal of full distributional recovery to a more robust,
geometry-first approach. We demonstrate a fundamental rate separation in the
low-noise limit, where information about the data manifold is encoded at a
significantly stronger scale ($\Theta(\sigma^{-2})$) than details about the
on-manifold distribution ($\Theta(1)$). This finding explains why models often
succeed at capturing the data support even with imperfect score estimates.
Building on this insight, we introduce Tempered Score (TS) Langevin dynamics, a
simple one-line modification that robustly targets the uniform volume measure on
the manifold, tolerating score errors up to $o(\sigma^{-2})$. This geometric
approach not only provides a more stable foundation for Bayesian inverse
problems but also, as shown in our experiments with models like Stable
Diffusion, empirically improves the diversity and fidelity of generated samples.

\textbf{Limitations and future work.}
Key limitations and future directions include:
\begin{enumerate*}[label=\alph*)]
    \item The implications for diffusion models are presently limited: we do not track cumulative error along the sampling trajectory; instead, we analyze a simplified setting that assumes access to the error of the final generated distribution.
    \item Our $L^{\infty}$ score–error assumption could potentially be relaxed to an $L^{2}$ bound, thereby aligning our theoretical framework with practical training objectives like denoising score matching (Fisher divergence) that minimize $L^{2}$ error.
    \item It remains to generalize the rate separation in score estimation into corresponding results on statistical sample complexity.
    \item Our analyses on the uniform sampling are in continuous time; we do not quantify discretization error arising in practical implementations.
    \item Our experiments are preliminary; we have not conducted a large-scale study with state-of-the-art diffusion models.
\end{enumerate*}

\section*{Acknowledgment}
The work is supported by ETH research grant, Swiss National Science Foundation (SNSF) Project Funding No. 200021-207343, and SNSF Starting Grant.

\bibliography{ref}

\begin{thebibliography}{60}
\providecommand{\natexlab}[1]{#1}
\providecommand{\url}[1]{\texttt{#1}}
\expandafter\ifx\csname urlstyle\endcsname\relax
  \providecommand{\doi}[1]{doi: #1}\else
  \providecommand{\doi}{doi: \begingroup \urlstyle{rm}\Url}\fi

\bibitem[Arts et~al.(2023)Arts, Garcia~Satorras, Huang, Zugner, Federici,
  Clementi, No{\'e}, Pinsler, and van~den Berg]{arts2023two}
Marloes Arts, Victor Garcia~Satorras, Chin-Wei Huang, Daniel Zugner, Marco
  Federici, Cecilia Clementi, Frank No{\'e}, Robert Pinsler, and Rianne van~den
  Berg.
\newblock Two for one: Diffusion models and force fields for coarse-grained
  molecular dynamics.
\newblock \emph{Journal of Chemical Theory and Computation}, 19\penalty0
  (18):\penalty0 6151--6159, 2023.

\bibitem[Azangulov et~al.(2024)Azangulov, Deligiannidis, and
  Rousseau]{azangulov2024convergence}
Iskander Azangulov, George Deligiannidis, and Judith Rousseau.
\newblock Convergence of diffusion models under the manifold hypothesis in
  high-dimensions.
\newblock \emph{arXiv preprint arXiv:2409.18804}, 2024.

\bibitem[Belkin \& Niyogi(2008)Belkin and Niyogi]{belkin2008towards}
Mikhail Belkin and Partha Niyogi.
\newblock Towards a theoretical foundation for laplacian-based manifold
  methods.
\newblock \emph{Journal of Computer and System Sciences}, 74\penalty0
  (8):\penalty0 1289--1308, 2008.

\bibitem[Bonnemain \& Ullmo(2019)Bonnemain and Ullmo]{bonnemain2019mean}
Thibault Bonnemain and Denis Ullmo.
\newblock Mean field games in the weak noise limit: A wkb approach to the
  fokker--planck equation.
\newblock \emph{Physica A: Statistical Mechanics and its Applications},
  523:\penalty0 310--325, 2019.

\bibitem[Bouchet \& Reygner(2016)Bouchet and
  Reygner]{bouchet2016generalisation}
Freddy Bouchet and Julien Reygner.
\newblock Generalisation of the eyring--kramers transition rate formula to
  irreversible diffusion processes.
\newblock In \emph{Annales Henri Poincar{\'e}}, volume~17, pp.\  3499--3532.
  Springer, 2016.

\bibitem[Brillouin(1926)]{brillouin1926mecanique}
L{\'e}on Brillouin.
\newblock La m{\'e}canique ondulatoire de schr{\"o}dinger; une m{\'e}thode
  g{\'e}n{\'e}rale de r{\'e}solution par approximations successives.
\newblock \emph{CR Acad. Sci}, 183\penalty0 (11):\penalty0 24--26, 1926.

\bibitem[Chen et~al.(2023)Chen, Chewi, Li, Li, Salim, and
  Zhang]{chen2023sampling}
Sitan Chen, Sinho Chewi, Jerry Li, Yuanzhi Li, Adil Salim, and Anru~R Zhang.
\newblock Sampling is as easy as learning the score: theory for diffusion
  models with minimal data assumptions.
\newblock In \emph{International Conference on Learning Representations}, 2023.

\bibitem[Coifman \& Lafon(2006)Coifman and Lafon]{coifman2006diffusion}
Ronald~R Coifman and St{\'e}phane Lafon.
\newblock Diffusion maps.
\newblock \emph{Applied and computational harmonic analysis}, 21\penalty0
  (1):\penalty0 5--30, 2006.

\bibitem[De~Bortoli(2022)]{de2022convergence}
Valentin De~Bortoli.
\newblock Convergence of denoising diffusion models under the manifold
  hypothesis.
\newblock \emph{Transactions on Machine Learning Research}, 2022.

\bibitem[De~Santi et~al.(2025)De~Santi, Vlastelica, Hsieh, Shen, He, and
  Krause]{de2025provable}
Riccardo De~Santi, Marin Vlastelica, Ya-Ping Hsieh, Zebang Shen, Niao He, and
  Andreas Krause.
\newblock Provable maximum entropy manifold exploration via diffusion models.
\newblock \emph{arXiv preprint arXiv:2506.15385}, 2025.

\bibitem[Gilbarg et~al.(1977)Gilbarg, Trudinger, Gilbarg, and
  Trudinger]{gilbarg1977elliptic}
David Gilbarg, Neil~S Trudinger, David Gilbarg, and NS~Trudinger.
\newblock \emph{Elliptic partial differential equations of second order},
  volume 224.
\newblock Springer, 1977.

\bibitem[Gong et~al.(2024)Gong, He, and Shen]{gong2024poincare}
Yun Gong, Niao He, and Zebang Shen.
\newblock Poincare inequality for local log-polyak-$\backslash$l ojasiewicz
  measures: Non-asymptotic analysis in low-temperature regime.
\newblock \emph{arXiv preprint arXiv:2501.00429}, 2024.

\bibitem[Graham \& T{\'e}l(1984)Graham and T{\'e}l]{graham1984weak}
R~Graham and T~T{\'e}l.
\newblock On the weak-noise limit of fokker-planck models.
\newblock \emph{Journal of statistical physics}, 35\penalty0 (5):\penalty0
  729--748, 1984.

\bibitem[Hein et~al.(2007)Hein, Audibert, and Luxburg]{hein2007graph}
Matthias Hein, Jean-Yves Audibert, and Ulrike~von Luxburg.
\newblock Graph laplacians and their convergence on random neighborhood graphs.
\newblock \emph{Journal of Machine Learning Research}, 8\penalty0 (6), 2007.

\bibitem[Hessel et~al.(2021)Hessel, Holtzman, Forbes, Bras, and
  Choi]{hessel2021clipscore}
Jack Hessel, Ari Holtzman, Maxwell Forbes, Ronan~Le Bras, and Yejin Choi.
\newblock Clipscore: A reference-free evaluation metric for image captioning.
\newblock \emph{arXiv preprint arXiv:2104.08718}, 2021.

\bibitem[Ho \& Salimans(2022)Ho and Salimans]{ho2022classifier}
Jonathan Ho and Tim Salimans.
\newblock Classifier-free diffusion guidance.
\newblock \emph{arXiv preprint arXiv:2207.12598}, 2022.

\bibitem[Ho et~al.(2020)Ho, Jain, and Abbeel]{ho2020denoising}
Jonathan Ho, Ajay Jain, and Pieter Abbeel.
\newblock Denoising diffusion probabilistic models.
\newblock \emph{Advances in neural information processing systems},
  33:\penalty0 6840--6851, 2020.

\bibitem[Holley \& Stroock(1987)Holley and Stroock]{holley1987logarithmic}
Richard Holley and Daniel Stroock.
\newblock Logarithmic sobolev inequalities and stochastic ising models.
\newblock \emph{Journal of Statistical Physics}, 46\penalty0 (5-6):\penalty0
  1159--1194, 1987.

\bibitem[Holmes(2012)]{holmes2012introduction}
Mark~H Holmes.
\newblock \emph{Introduction to perturbation methods}, volume~20.
\newblock Springer Science \& Business Media, 2012.

\bibitem[Hwang(1980)]{hwang1980laplace}
Chii-Ruey Hwang.
\newblock Laplace's method revisited: weak convergence of probability measures.
\newblock \emph{The Annals of Probability}, pp.\  1177--1182, 1980.

\bibitem[Hyv{\"a}rinen \& Dayan(2005)Hyv{\"a}rinen and
  Dayan]{hyvarinen2005estimation}
Aapo Hyv{\"a}rinen and Peter Dayan.
\newblock Estimation of non-normalized statistical models by score matching.
\newblock \emph{Journal of Machine Learning Research}, 6\penalty0 (4), 2005.

\bibitem[Janati et~al.(2024)Janati, Moufad, Durmus, Moulines, and
  Olsson]{janati2024divide}
Yazid Janati, Badr Moufad, Alain Durmus, Eric Moulines, and Jimmy Olsson.
\newblock Divide-and-conquer posterior sampling for denoising diffusion priors.
\newblock \emph{Advances in Neural Information Processing Systems},
  37:\penalty0 97408--97444, 2024.

\bibitem[Jost(2005)]{jost2005riemannian}
J{\"u}rgen Jost.
\newblock \emph{Riemannian geometry and geometric analysis}.
\newblock Springer, 2005.

\bibitem[Kadkhodaie \& Simoncelli(2020)Kadkhodaie and
  Simoncelli]{kadkhodaie2020solving}
Zahra Kadkhodaie and Eero~P Simoncelli.
\newblock Solving linear inverse problems using the prior implicit in a
  denoiser.
\newblock \emph{arXiv preprint arXiv:2007.13640}, 2020.

\bibitem[Karras et~al.(2022)Karras, Aittala, Aila, and
  Laine]{karras2022elucidating}
Tero Karras, Miika Aittala, Timo Aila, and Samuli Laine.
\newblock Elucidating the design space of diffusion-based generative models.
\newblock \emph{Advances in neural information processing systems},
  35:\penalty0 26565--26577, 2022.

\bibitem[Kramers(1926)]{kramers1926wellenmechanik}
Hendrik~Anthony Kramers.
\newblock Wellenmechanik und halbzahlige quantisierung.
\newblock \emph{Zeitschrift f{\"u}r Physik}, 39\penalty0 (10):\penalty0
  828--840, 1926.

\bibitem[{\L}api{\'n}ski(2019)]{lapinski2019multivariate}
Tomasz~M {\L}api{\'n}ski.
\newblock Multivariate laplace’s approximation with estimated error and
  application to limit theorems.
\newblock \emph{Journal of Approximation Theory}, 248:\penalty0 105305, 2019.

\bibitem[Laumont et~al.(2022)Laumont, Bortoli, Almansa, Delon, Durmus, and
  Pereyra]{laumont2022bayesian}
R{\'e}mi Laumont, Valentin~De Bortoli, Andr{\'e}s Almansa, Julie Delon, Alain
  Durmus, and Marcelo Pereyra.
\newblock Bayesian imaging using plug \& play priors: when langevin meets
  tweedie.
\newblock \emph{SIAM Journal on Imaging Sciences}, 15\penalty0 (2):\penalty0
  701--737, 2022.

\bibitem[Lee et~al.(2023)Lee, Lu, and Tan]{lee2023convergence}
Holden Lee, Jianfeng Lu, and Yixin Tan.
\newblock Convergence of score-based generative modeling for general data
  distributions.
\newblock In \emph{International Conference on Algorithmic Learning Theory},
  pp.\  946--985. PMLR, 2023.

\bibitem[Leobacher \& Steinicke(2021)Leobacher and
  Steinicke]{leobacher2021existence}
Gunther Leobacher and Alexander Steinicke.
\newblock Existence, uniqueness and regularity of the projection onto
  differentiable manifolds.
\newblock \emph{Annals of global analysis and geometry}, 60\penalty0
  (3):\penalty0 559--587, 2021.

\bibitem[Liu et~al.(2025)Liu, Zhang, and Li]{liu2025improving}
Zichen Liu, Wei Zhang, and Tiejun Li.
\newblock Improving the euclidean diffusion generation of manifold data by
  mitigating score function singularity.
\newblock \emph{arXiv preprint arXiv:2505.09922}, 2025.

\bibitem[Loaiza-Ganem et~al.(2024)Loaiza-Ganem, Ross, Hosseinzadeh, Caterini,
  and Cresswell]{loaizadeep}
Gabriel Loaiza-Ganem, Brendan~Leigh Ross, Rasa Hosseinzadeh, Anthony~L
  Caterini, and Jesse~C Cresswell.
\newblock Deep generative models through the lens of the manifold hypothesis: A
  survey and new connections.
\newblock \emph{Transactions on Machine Learning Research}, 2024.

\bibitem[Lu et~al.(2023)Lu, Wang, and Bal]{lu2023mathematical}
Yubin Lu, Zhongjian Wang, and Guillaume Bal.
\newblock Mathematical analysis of singularities in the diffusion model under
  the submanifold assumption.
\newblock \emph{arXiv preprint arXiv:2301.07882}, 2023.

\bibitem[Lyu et~al.(2025)Lyu, Nguyen, Qian, and Tong]{lyu2025resolving}
Yang Lyu, Tan~Minh Nguyen, Yuchun Qian, and Xin~T Tong.
\newblock Resolving memorization in empirical diffusion model for manifold data
  in high-dimensional spaces.
\newblock \emph{arXiv preprint arXiv:2505.02508}, 2025.

\bibitem[Maes et~al.(2009)Maes, Neto{\v{c}}n{\`y}, and
  Shergelashvili]{maes2009nonequilibrium}
Christian Maes, Karel Neto{\v{c}}n{\`y}, and Bidzina~M Shergelashvili.
\newblock Nonequilibrium relation between potential and stationary distribution
  for driven diffusion.
\newblock \emph{Physical Review E—Statistical, Nonlinear, and Soft Matter
  Physics}, 80\penalty0 (1):\penalty0 011121, 2009.

\bibitem[Maier \& Stein(1997)Maier and Stein]{maier1997limiting}
Robert~S Maier and Daniel~L Stein.
\newblock Limiting exit location distributions in the stochastic exit problem.
\newblock \emph{SIAM Journal on Applied Mathematics}, 57\penalty0 (3):\penalty0
  752--790, 1997.

\bibitem[Majerski(2015)]{majerski2015simple}
Piotr Majerski.
\newblock Simple error bounds for the multivariate laplace approximation under
  weak local assumptions.
\newblock \emph{arXiv preprint arXiv:1511.00302}, 2015.

\bibitem[Matkowsky \& Schuss(1977)Matkowsky and Schuss]{matkowsky1977exit}
Bernard~J Matkowsky and Zeev Schuss.
\newblock The exit problem for randomly perturbed dynamical systems.
\newblock \emph{SIAM Journal on Applied Mathematics}, 33\penalty0 (2):\penalty0
  365--382, 1977.

\bibitem[Menz \& Schlichting(2014)Menz and Schlichting]{AOP}
Georg Menz and Andr{\'e} Schlichting.
\newblock Poincar{\'e} and logarithmic sobolev inequalities by decomposition of
  the energy landscape.
\newblock \emph{The Annals of Probability}, 42\penalty0 (5):\penalty0 1809,
  2014.

\bibitem[Milnor \& Stasheff(1974)Milnor and Stasheff]{milnor1974characteristic}
John~Willard Milnor and James~D Stasheff.
\newblock \emph{Characteristic classes}.
\newblock Number~76. Princeton university press, 1974.

\bibitem[Munkres(2000)]{munkres2000topology}
James~Raymond Munkres.
\newblock \emph{Topology}.
\newblock Prentice Hall, 2nd edition, 2000.

\bibitem[Pavlova \& Wei(2025)Pavlova and Wei]{pavlova2025diffusion}
Elizabeth Pavlova and Xue-Xin Wei.
\newblock Diffusion models under low-noise regime.
\newblock \emph{arXiv preprint arXiv:2506.07841}, 2025.

\bibitem[Pesme et~al.(2025)Pesme, Meanti, Arbel, and Mairal]{pesme2025map}
Scott Pesme, Giacomo Meanti, Michael Arbel, and Julien Mairal.
\newblock Map estimation with denoisers: Convergence rates and guarantees.
\newblock \emph{arXiv preprint arXiv:2507.15397}, 2025.

\bibitem[Pidstrigach(2022)]{pidstrigach2022score}
Jakiw Pidstrigach.
\newblock Score-based generative models detect manifolds.
\newblock \emph{Advances in Neural Information Processing Systems},
  35:\penalty0 35852--35865, 2022.

\bibitem[Raja et~al.(2025)Raja, {\v{S}}{\'\i}pka, Psenka, Kreiman, Pavelka, and
  Krishnapriyan]{raja2025action}
Sanjeev Raja, Martin {\v{S}}{\'\i}pka, Michael Psenka, Tobias Kreiman, Michal
  Pavelka, and Aditi~S Krishnapriyan.
\newblock Action-minimization meets generative modeling: Efficient transition
  path sampling with the onsager-machlup functional.
\newblock \emph{arXiv preprint arXiv:2504.18506}, 2025.

\bibitem[Rey-Bellet \& Spiliopoulos(2015)Rey-Bellet and
  Spiliopoulos]{rey2015irreversible}
Luc Rey-Bellet and Konstantinos Spiliopoulos.
\newblock Irreversible langevin samplers and variance reduction: a large
  deviations approach.
\newblock \emph{Nonlinearity}, 28\penalty0 (7):\penalty0 2081, 2015.

\bibitem[Rombach et~al.(2022)Rombach, Blattmann, Lorenz, Esser, and
  Ommer]{Rombach_2022_CVPR}
Robin Rombach, Andreas Blattmann, Dominik Lorenz, Patrick Esser, and Bj\"orn
  Ommer.
\newblock High-resolution image synthesis with latent diffusion models.
\newblock In \emph{Proceedings of the IEEE/CVF Conference on Computer Vision
  and Pattern Recognition (CVPR)}, pp.\  10684--10695, June 2022.

\bibitem[Saharia et~al.(2022)Saharia, Chan, Saxena, Li, Whang, Denton,
  Ghasemipour, Gontijo~Lopes, Karagol~Ayan, Salimans,
  et~al.]{saharia2022photorealistic}
Chitwan Saharia, William Chan, Saurabh Saxena, Lala Li, Jay Whang, Emily~L
  Denton, Kamyar Ghasemipour, Raphael Gontijo~Lopes, Burcu Karagol~Ayan, Tim
  Salimans, et~al.
\newblock Photorealistic text-to-image diffusion models with deep language
  understanding.
\newblock \emph{Advances in neural information processing systems},
  35:\penalty0 36479--36494, 2022.

\bibitem[Saremi et~al.(2023)Saremi, Srivastava, and Bach]{saremi2023universal}
Saeed Saremi, Rupesh~Kumar Srivastava, and Francis Bach.
\newblock Universal smoothed score functions for generative modeling.
\newblock \emph{arXiv preprint arXiv:2303.11669}, 2023.

\bibitem[Song et~al.(2020)Song, Meng, and Ermon]{song2020denoising}
Jiaming Song, Chenlin Meng, and Stefano Ermon.
\newblock Denoising diffusion implicit models.
\newblock \emph{arXiv preprint arXiv:2010.02502}, 2020.

\bibitem[Song \& Ermon(2019)Song and Ermon]{song2019generative}
Yang Song and Stefano Ermon.
\newblock Generative modeling by estimating gradients of the data distribution.
\newblock \emph{Advances in neural information processing systems}, 32, 2019.

\bibitem[Song et~al.(2021)Song, Sohl-Dickstein, Kingma, Kumar, Ermon, and
  Poole]{songscore}
Yang Song, Jascha Sohl-Dickstein, Diederik~P Kingma, Abhishek Kumar, Stefano
  Ermon, and Ben Poole.
\newblock Score-based generative modeling through stochastic differential
  equations.
\newblock In \emph{ICLR}, 2021.

\bibitem[Stanczuk et~al.(2024)Stanczuk, Batzolis, Deveney, and
  Sch{\"o}nlieb]{stanczuk2024diffusion}
Jan~Pawel Stanczuk, Georgios Batzolis, Teo Deveney, and Carola-Bibiane
  Sch{\"o}nlieb.
\newblock Diffusion models encode the intrinsic dimension of data manifolds.
\newblock In \emph{Forty-first International Conference on Machine Learning},
  2024.

\bibitem[Tang \& Yang(2024)Tang and Yang]{tang2024adaptivity}
Rong Tang and Yun Yang.
\newblock Adaptivity of diffusion models to manifold structures.
\newblock In \emph{International Conference on Artificial Intelligence and
  Statistics}, pp.\  1648--1656. PMLR, 2024.

\bibitem[Venkatakrishnan et~al.(2013)Venkatakrishnan, Bouman, and
  Wohlberg]{venkatakrishnan2013plug}
Singanallur~V Venkatakrishnan, Charles~A Bouman, and Brendt Wohlberg.
\newblock Plug-and-play priors for model based reconstruction.
\newblock In \emph{2013 IEEE global conference on signal and information
  processing}, pp.\  945--948. IEEE, 2013.

\bibitem[Ventura et~al.(2024)Ventura, Achilli, Silvestri, Lucibello, and
  Ambrogioni]{ventura2024manifolds}
Enrico Ventura, Beatrice Achilli, Gianluigi Silvestri, Carlo Lucibello, and
  Luca Ambrogioni.
\newblock Manifolds, random matrices and spectral gaps: The geometric phases of
  generative diffusion.
\newblock \emph{arXiv preprint arXiv:2410.05898}, 2024.

\bibitem[Vincent(2011)]{vincent2011connection}
Pascal Vincent.
\newblock A connection between score matching and denoising autoencoders.
\newblock \emph{Neural computation}, 23\penalty0 (7):\penalty0 1661--1674,
  2011.

\bibitem[Wentzel(1926)]{wentzel1926verallgemeinerung}
Gregor Wentzel.
\newblock Eine verallgemeinerung der quantenbedingungen f{\"u}r die zwecke der
  wellenmechanik.
\newblock \emph{Zeitschrift f{\"u}r Physik}, 38\penalty0 (6):\penalty0
  518--529, 1926.

\bibitem[Weyl(1939)]{weyl1939volume}
Hermann Weyl.
\newblock On the volume of tubes.
\newblock \emph{American Journal of Mathematics}, 61\penalty0 (2):\penalty0
  461--472, 1939.

\bibitem[Willard(2012)]{willard2012general}
Stephen Willard.
\newblock \emph{General topology}.
\newblock Courier Corporation, 2012.

\end{thebibliography}
\bibliographystyle{iclr2026_conference}

\appendix

\crefalias{section}{appendix}
\crefalias{subsection}{appendix}

\section{Additional Notation and Preliminaries}
\label{sec:apx_prelim}

In this section, we provide some notation and preliminaries complementary to
\Cref{sec:prelim}.

We denote by $W_t$ a standard Brownian motion, with its dimension clear from
context. The Gaussian density with mean $\mu$ and covariance $\Sigma$, evaluated
at $x$, is written as $\cN(x \mid \mu, \Sigma)$. The symbol $\ast$ denotes the
convolution operator. We use $\propto$ to indicate proportionality, i.e., that
the left-hand side and right-hand side are equal up to a constant factor.  
For a set $S$, we write $\overline{S}$ for its closure, $\partial S$ for its
boundary, and $S^c$ for its complement. Throughout the paper, by the term
\emph{limiting distribution} or by convergence of a distribution/density
function, we mean convergence of the corresponding measures in the weak sense.

\subsection{The Manifold Hypothesis}
We outline
few notations and standard results from differential geometry.
By the tubular neighborhood
theorem~\citep{milnor1974characteristic,weyl1939volume}, there exists
$\epsilon > 0$ such that the normal tube
\[
T_{\cM}(\epsilon) := \{x \in \bR^d : \operatorname{dist}(x,\cM) < \epsilon\}.
\]
admits local $C^4$ coordinate
\[
\Phi : U \times R \to T_{\cM}(\epsilon), \quad \text{where} \quad U \subset \bR^n, R := \{r \in \bR^{d-n} : \|r\| < \epsilon\},
\]
such that $\Phi$ is a diffeomorphism mapping from local coordinates to ambient Euclidean space.
With this result, we can then work with local coordinates to describe the manifold.
For notational simplicity, we work with a single chart and suppress indices:
$u \in U$ denote tangential coordinates and $r \in R$ denote normal coordinates.
The slice $r = 0$ corresponds to points on $\cM$, and we write
$\Phi(u) := \Phi(u,0)$.
Let $J(u,r)$ denote the Jacobian of $\Phi(u,r)$ with respect to $(u,r)$, i.e., 
$J(u,r) = \partial \Phi(u,r) / \partial(u,r)$.
Furthermore, let $g(u)$ denote the Riemannian metric tensor of the manifold
$\cM$, defined as $g(u) := J(u,0)^\top J(u,0)$. Intuitively, the Riemannian
metric tensor gives a way to measure lengths and angles of the manifold geometry.

\section{Proofs of Main Theorems}
\label{sec:proof_main_theorems}

In this section, we prove the main theorems of the paper. We begin by
developing a general framework for characterizing the limiting distribution
when the density admits a specific form. This framework will then be applied
to establish the results in \Cref{sec:scale}, where such a density form was
assumed.

The results in \Cref{sec:recover_uniform} require a different approach, since
no explicit form of the density is available. In this case, we employ the WKB
approximation to obtain an approximate stationary distribution, which we then
substitute into the general framework to derive the limiting distribution.

\subsection{A General Framework for the Convergence of the Limiting Distribution}
\label{subsec:general_proof}

In this subsection, we will establish a general framework for the limiting
distribution of density proportional to
\begin{align} \label{eq:general_density}
  \exp\left(- \left(f_{\theta}(x)\right)  / \theta\right), \quad
  \text{with} \quad  f_{\theta}(x) = f_0(x) + \theta f_1(x) + \hat f(x, \theta),
\end{align}
where $f_0$'s minimizer is on the manifold $\cM$ and $\hat f(x, \theta)$ is a 
perturbation that is uniformly $o(\theta)$ so that it does not affect the
limiting distribution.
This general result is stated in \Cref{thm:dist_f_theta}.
Our main results fall into this framework by letting $\theta = \sigma^2$
for \Cref{thm:recover_pdata} and $\theta = \sigma^{2-\alpha}$ for \Cref{thm:recover_uniform}.

In all cases the theorems we will prove later, the density will concentrate on
the tubular neighborhood of $M$, i.e., $T_{\cM}(\epsilon)$. Therefore, we will
discuss the lemmas and intermediate results in such a neighborhood and use local
coordinates $(u, r)$. The notations used can be found in \Cref{sec:prelim}.
When we use local coordinates, we assume the discussion is in the closure of $T_{\cM}(\epsilon)$.
We define the local coordinate versions of the functions:
$f_{\theta}(u, r) \coloneqq f_{\theta}(\Phi(u, r))$,
$f_0(u, r) \coloneqq f_0(\Phi(u, r))$, $f_1(u, r) \coloneqq f_1(\Phi(u, r))$,
and $\hat f(u, r, \theta) \coloneqq \hat f(\Phi(u, r), \theta)$.

Our assumptions are stated as follows. 
\begin{assumption} \label{assumption:f}
  We assume that
  \begin{enumerate}
    \item $\cM \subset \bR^d$ is a compact $C^4$ manifold without boundary
      with dimension $n < d$.
    \item $\cM = \argmin_{x \in T_{\cM}(\epsilon)} f_0(x)$. In addition, we assume that
      there exists $0 < \hat \epsilon < \epsilon$ such that
      $\inf_{x \in T_{\cM}(\epsilon) \backslash \overline{T_{\cM}(\hat \epsilon)}} f_0(x) - \min_{x \in T_{\cM}(\epsilon)}f_0(x)$
      is bounded away from zero.
    \item The absolute value of $\hat f(u, r, \theta)$ is $o(\theta)$ as $\theta \to 0$ uniformly for all $u \in U$ and $\norm{r} < \epsilon$.
  \item $f_0 \geq 0$ is $C^3$, $f_1$ is $C^1$, and $f_\theta$ is continuous on coordinates $(u, r)$
    for all $u \in U$ and $\norm{r} \leq \epsilon$, i.e., in the closure of $T_{\cM}(\epsilon)$.
  \item Further, we assume that the smallest eigenvalue of $\frac{\partial^2 f_{0}}{\partial r^2} (u, r)$
    is uniformly bounded away from zero for all $u \in U$ and $\norm{r} < \epsilon$.
  \end{enumerate}
\end{assumption}

\begin{remark}[Compactness of the manifold implies boundedness of gradients.] \label{remark:boundedness_in_compact_manifold}
Consider $f \in C^k(\overline{T_{\cM}(\epsilon)})$. In local coordinates $(u,r)$ induced by a tubular
atlas, we write $f(u,r) \coloneqq f(\Phi(u,r))$.
Since $\cM$ is compact,
one can choose a finite atlas with precompact coordinate domains.
Let the cover be $\{U_i\}$.
By the Shrinking Lemma (\citet[Theorem~32.3]{munkres2000topology} combined with 
\citet[Theorem~15.10]{willard2012general}), there exist open subsets $\{V_i\}$ with
$\overline{V_i} \subset U_i$ such that $\{V_i\}$ still forms a cover.
We use these $\{V_i\}$ as the new atlas.
The transition maps
$\Phi$ and their derivatives are then bounded on these sets (since $\overline{V_i}$
is compact), and by the chain rule the same holds for $f(u,r)$ and its derivatives up to order $k$.
Thus, throughout our arguments we may freely assume uniform boundedness
of such derivatives without loss of generality.
The same reasoning applies to $\pdata$, we can use the same constructed atlas
such that $\pdata$ is uniformly lower and upper bounded, and gradients of
$\pdata$ are uniformly upper bounded.
\end{remark}

During our proofs, we will frequently use Laplace's method for integrals.
We adapt the error estimate from \citet{lapinski2019multivariate} as follows.
\begin{corollary}[{Theorem~2 of \citet{lapinski2019multivariate}}]
  \label[corollary]{corollary:laplace}
  Let $\Omega\subset\bR^{m}$ be an open set and let $\Omega'\subset\Omega$ be a closed ball.
  Let $c_{1} \coloneqq \vol(\Omega')$.
  Let $F,g:\Omega\to\bR$ with the following assumptions:
  \begin{enumerate}
    \item $F|_{\Omega'}\in C^{3}(\Omega')$ and $F\ge 0$ on $\Omega$.
    There is a unique minimizer $x^{*}\in\operatorname{int}(\Omega')$ of $F$ on $\Omega$.
    Define
    \[
      m_{1}\;:=\;\inf_{x\in\Omega\setminus\Omega'}\big\{F(x)-F(x^{*})\big\}>0,
      \qquad
      m_{2}\;:=\;\inf_{x\in\Omega'}\lambda_{\min}\big(\nabla^{2}F(x)\big)>0.
    \]
    Let
    \[
      c_{2}\;:=\;\sup_{x\in\Omega'}\|\nabla^{2}F(x)\|,
      \qquad
      c_{3}\;:=\;\sup_{x\in\Omega'}\|\nabla^{3}F(x)\|.
    \]
    \item $g|_{\Omega'}\in C^{1}(\Omega')$ and $\int_{\Omega}|g(x)|\,dx<\infty$.
    Let
    \[
      c_{4}\;:=\;\sup_{x\in\Omega'}|g(x)|,\qquad
      c_{5}\;:=\;\sup_{x\in\Omega'}\|\nabla g(x)\|,\qquad
      c_{6}\;:=\;\int_{\Omega}|g(x)|\,dx.
    \]
  \end{enumerate}
  Then, for every $\theta>0$,
  \[
    \int_{\Omega} g(x)\,e^{-F(x)/\theta}\,dx
    \;=\;
    \exp(- F(x^*) / \theta) \frac{(2 \pi \theta)^{m/2}}{\sqrt{\abs*{\nabla^2 F(x^*)}}}
    \left(g(x^*) + h(\theta) \right),
  \]
  where $\abs{h(\theta)}$ can be upper bounded by a function of $\left(c_{1},\dots,c_{6},m_{1},m_{2}\right)$.
  Moreover, $h(\theta)=O(\sqrt{\theta})$ as $\theta\to0$.
  The $O(\sqrt{\theta})$ is \emph{uniform} over any class of pairs $(F,g)$ for which
  $c_{1},\dots,c_{6}$ are bounded above and $m_{1},m_{2}$ are bounded below by
  strictly positive constants uniformly over the class.
\end{corollary}

\begin{proof}
  The result follows directly from \citet[Theorem~2]{lapinski2019multivariate}.
\end{proof}

To show the convergence of the distribution to a distribution on the manifold,
a key step is to integrate out the normal direction so as to obtain a
distribution on $u$, such as what \citet{hwang1980laplace} did.
The following lemma proves Laplace's type of result for integrating out $r$.
\begin{lemma} \label[lemma]{lemma:limit_f_theta}
  Assume \Cref{assumption:f}, and let $h(x): \bR^d \to \bR$  be
  $C^1$ and uniformly bounded in $T_{\cM}(\epsilon)$. Define  $h(u, r): = h(\Phi(u, r))$.
  Then we have
  \begin{align*}
    &\fakeeq \int_{\norm{r} < \epsilon} \exp\left( - \frac{f_{\theta}(u, r)}{\theta} \right) h(u, r) dr \\
    &= \exp\left( - \frac{f_0(u, 0)}{\theta} \right) \exp\left(-f_1(u, 0)\right) \frac{\left(2 \pi \theta\right)^{(d-n)/2}}{\sqrt{\abs*{ \frac{\partial^2 f_{0}}{\partial r^2} (u, 0) }}} 
    \left(h(u, 0) + o(1)\right),
  \end{align*}
  where the $o(1)$ term is uniform for $u$.
\end{lemma}
\begin{proof}
  We have that
  \begin{align*}
    &\fakeeq \int_{\norm{r} < \epsilon} \exp\left( - \frac{f_{\theta}(u, r)}{\theta} \right) h(u, r) dr \\
    &= \int_{\norm{r} < \epsilon} \exp\left( - \frac{f_0(u, r)}{\theta} \right) \exp\left(-f_1(u, r)\right) h(u, r) 
    \left(\exp\left(- \frac{\hat f\left(u, r, \theta\right)}{\theta}\right) \right)dr \\
    &= \int_{\norm{r} < \epsilon} \exp\left( - \frac{f_0(u, r)}{\theta} \right) \exp\left(-f_1(u, r)\right) h(u, r) dr + \\
    &\fakeeq \int_{\norm{r} < \epsilon} \exp\left( - \frac{f_0(u, r)}{\theta} \right) \exp\left(-f_1(u, r)\right) h(u, r)
    \left(\exp\left(- \frac{\hat f\left(u, r, \theta\right)}{\theta}\right) - 1 \right) dr.
  \end{align*}
  For the first term, we can directly apply \Cref{corollary:laplace}
  with $F(r) = f_0(u, r)$, $g(r) = \exp(-f_1(u, r))h(u,r)$, and
  $\Omega'$ being the ball $\{r \mid \norm{r} \leq \hat \epsilon \}$.
  Define
  \begin{align*}
    J = \exp\left( - \frac{f_0(u, 0)}{\theta} \right) \exp\left(-f_1(u, 0)\right) \frac{\left(2 \pi \theta\right)^{(d-n)/2}}{\sqrt{\abs*{ \frac{\partial^2 f_{0}}{\partial r^2} (u, 0) }}}.
  \end{align*}
  The first term can be approximated as $J \left(h(u, 0) + o(1)\right)$.
  The boundedness of the quantities in \Cref{corollary:laplace} will be discussed later.
  The second term can be upper bounded by
  \begin{align*}
    &\fakeeq \sup_r \abs{h(u, r)} \cdot \sup_{r} \abs*{\exp\left(-\frac{\hat f(u, r, \theta)}{\theta} \right) - 1} \int_{\norm{r} < \epsilon} \exp\left( - \frac{f_0(u, r)}{\theta} \right) \exp\left(-f_1(u, r)\right) dr \\
    &= o(1) J(1 + o(1)) = o(1) J,
  \end{align*}
  where we used \Cref{corollary:laplace} for the integral.
  The lower bound can be obtained similarly.
  The result follows.

  Regarding the uniform boundedness of the quantities in \Cref{corollary:laplace},
  $\{c\}_1^5$ is uniformly bounded by the compactness of the manifold.
  The constant $c_6$ is uniformly bounded by our assumption on $h$.
  The uniform lower bounds of $m_1$ and $m_2$ is guaranteed by \Cref{assumption:f}.
\end{proof}

Next, we will prove that the support of the limiting distribution will concentrate
on the minimizers of the leading term.
Previously, we considered $f_{\theta}$ consisting of $f_0 + \Theta(\theta) + o(\theta)$.
Next, we will show that as long as $f_{\theta}$ is $f_0 + o(1)$, the concentration
on $f_0$'s minimizers will happen.

\begin{lemma} \label[lemma]{lemma:limit_support}
Let $f_{\theta}(x) = f_0(x) + \tilde{f}(x, \theta)$, such that $\exp(- f_{\theta}(x) / \theta)$
is a normalized density function on $\bR^d$.
Suppose $\cM$ is a connected and compact $C^4$ manifold without boundary. Assume that:
\begin{enumerate}
  \item $f_0(x)$ is continuous with $\argmin_{x \in \overline{T_{\cM}(\epsilon)}} f_0(x) = \cM$ and
    $\min_{x \in x \in \overline{T_{\cM}(\epsilon)}} f_0(x) = 0$.
  \item $\tilde f(x, \theta)$ is continuous and uniformly $o(1)$ as $\theta \to 0$
    for all $x \in \overline{T_{\cM}(\epsilon)}$.
  \item The density concentrates in $T_{\cM}(\epsilon)$, i.e.,
  \[
    \lim_{\theta \to 0} \int_{T_{\cM}(\epsilon)} \exp\left(-\frac{f_{\theta}(x)}{\theta}\right) dx = 1.
  \]
\end{enumerate}
For any $\eta > 0$, define the set $C_{\eta} = \{x \mid f_0(x) > \eta\}$. Then,
\begin{align*}
  \int_{C_{\eta} \cup T_{\cM}(\epsilon)^c} \exp(-f_{\theta}(x)/\theta) \, dx \to 0 \quad \text{as} \quad \theta \to 0.
\end{align*}
If in addition, $\exp(-f_{\theta}(x)/\theta)$ converges weakly to a distribution as $\theta \to 0$, the support of the limiting distribution is contained in $\cM$.
\end{lemma}
\begin{proof}
  Since we have that $\int_{T_{\cM}(\epsilon)} \exp(-f_{\theta}(x)/\theta) dx \to 1$,
  for the first result, it suffices to show that $\int_{T_{\cM}(\epsilon) \cap C_{\eta}} \exp(-f_{\theta}(x)/\theta) dx \to 0$.
  According to the assumptions, we have that
  for any $\delta$, $\exists \theta_0$, such that $\forall \theta < \theta_0$,
  $\abs{\tilde f(x, \theta)} < \delta$.
  Therefore, we have
  \begin{align*}
    \int_{T_{\cM}(\epsilon) \cap C_{\eta}} \exp(-f_{\theta}(x)/\theta) dx
    \leq \int_{T_{\cM}(\epsilon) \cap C_{\eta}} \exp((-\eta + \delta)/\theta) dx
    \leq \vol(T_{\cM}(\epsilon)) \exp((-\eta + \delta)/\theta).
  \end{align*}
  We choose $\delta = \eta/2$, then the right-hand side goes to zero as $\theta \to 0$.

  Let the limiting measure be $P$, and $P_{\theta}$ be the probability measure
  corresponding to the density $\exp(-f_{\theta}(x)/\theta)$.
  Since $C_{\eta}$ is an open set, we have that
  \begin{align*}
    P(C_{\eta}) \leq \liminf_{\theta \to 0} P_{\theta}(C_{\eta}) = 0.
  \end{align*}
  We also have that
  \begin{align*}
    P\left(\overline{T_{\cM}(\epsilon)}^c\right) \leq 
    \liminf_{\theta \to 0} P_{\theta}\left(\overline{T_{\cM}(\epsilon)}^c\right)
    \leq \liminf_{\theta \to 0} P_{\theta}\left(T_{\cM}(\epsilon)^c\right) = 0.
  \end{align*}
  Denote $C \coloneqq \cM^c$.
  We have that $C = \cup_{m = 1}^{\infty} C_{1/m} \cup \overline{T_{\cM}(\epsilon)}^c$.
  Then we have
  \begin{align*}
    P(C) \leq \sum_{m=1}^{\infty} P(C_{1/m}) + P\left(\overline{T_{\cM}(\epsilon)}^c\right) = 0.
  \end{align*}
  which concludes the proof.
\end{proof}

\begin{theorem} \label{thm:dist_f_theta}
  Assume \Cref{assumption:f}.
  Define
  \begin{align*}
    \pi_{\theta} (x) \propto \exp\left( - \frac{f_{\theta}(x)}{\theta} \right),
  \end{align*}
  Assume that $1 - \int_{x \in T_{\cM}(\epsilon)} \pi_{\theta}(x) dx \to 0$
   as $\theta \to 0$.
  Then we have that
  as $\theta \to 0$, $\pi_{\theta}$ converges weakly to the following distribution:
  \begin{align*}
    \pi(u) = \frac{\exp(-f_1(u, 0))\abs*{\frac{\partial^2 f_0(u, 0)}{\partial r^2}}^{-1/2} d \cM(u) / du}{\int_{\cM} \exp(-f_1(u, 0)) \abs*{\frac{\partial^2 f_0(u, 0)}{\partial r^2}}^{-1/2} d \cM(u) / du},
  \end{align*}
  where $d \cM$ is the intrinsic measure on the manifold $\cM$,
  i.e., $d \cM(u) = \abs{g(u)}^{1/2} du$, and $du$ is the Lebesgue measure on the local parameterization domain $U$.
\end{theorem}
\begin{proof}
  The proof follows the same as the proof in \citet[Theorem 3.1]{hwang1980laplace}.
  The only difference is that we replace the estimate of \citet[Equation~(3.2)]{hwang1980laplace}
  with our \Cref{lemma:limit_f_theta}.
  Note that the Q in \citet[Theorem 3.1]{hwang1980laplace} is assumed as
  a probability measure, thus $f$ (in his notation) integrates to one.
  However, the proof technique of \citet[Theorem 3.1]{hwang1980laplace}
  remains valid even if $f$ is not a probability density, so applying to
  our case.
\end{proof}

\subsection{Proof for \Cref{thm_main_informal}}

The remaining the proof is to expand the true log-density w.r.t. $\sigma$,
analyze the error of the learned log-density,
and then to plug in the result obtained from \Cref{subsec:general_proof}.

\begin{theorem} \label{lemma:limit_V}
  Assume \Cref{assumption:manifold,assumption:pdata} holds.
  Suppose $x \in T_{\cM}(\epsilon)$. Then we have that
  \begin{align*}
  \log p_{\sigma}^\mathrm{VE}(x)
  &= -\frac{1}{2 \sigma^2} \norm{ x - \Pm(x) }^2 + \log \pdata(\Phi^{-1}(\Pm(x)))
    - \frac{d-n}{2} \log(2\pi \sigma^2) - \\ 
    &\fakeeq  \log \sqrt{\abs*{\hat H(\Phi^{-1}(\Pm(x)), x)}} 
    + \hat p^\mathrm{VE}(x, \sigma), \\
  \log p_{\sigma}^\mathrm{VP}(x)
  &= -\frac{1}{2 \sigma^2} \norm{ x - \Pm(x) }^2 + \log \pdata(\Phi^{-1}(\Pm(x)))
    - \frac{d-n}{2} \log(2\pi \sigma^2) - \\ 
    &\fakeeq \log \sqrt{\abs*{\hat H(\Phi^{-1}(\Pm(x)), x)}} 
    - \frac{1}{2} \inp*{\Pm(x)}{x - \Pm(x)}
    + \hat p^\mathrm{VP}(x, \sigma),
  \end{align*}
  where $\hat p^\mathrm{VE}(x, \sigma)$ and $\hat p^\mathrm{VP}(x, \sigma)$ are
  functions that are $o(1)$ uniformly for $x \in T_{\cM}(\epsilon)$.
  The matrix $\hat H(u, x)$ is such that
  \begin{align*}
    \hat H(u, x)_{i, j} = \inp*{\frac{\partial^2 \Phi(u)}{\partial u_i \partial u_j}}{\Phi(u) - x}
  + \inp*{\frac{\partial \Phi(u)}{\partial u_i}}{\frac{\partial \Phi(u)}{\partial u_j}}.
  \end{align*}
\end{theorem}
\begin{proof}
  We can apply \Cref{corollary:laplace} as an error estimate for Laplace's
  method, to the integral in $p_{\sigma}$.
  The minimizer of $F(u)$ is $\Phi^{-1}(\Pm(x))$ for both VE and VP.

  We first consider the case of VE.
  By letting $F(u) = \| x - \Phi(u) \|^2 / 2$, $g(u) = \pdata(u)$ and $\theta = \sigma^2$
  we can obtain that
  \begin{align} \label{eq:p_sigma_ve}
    p_{\sigma} (x) = \exp\left(- \frac{\norm{ x - \Pm(x) }^2}{2\sigma^2}\right)
  \frac{\left(2\pi \sigma^2\right)^{(n-d)/2}}{\sqrt{\abs*{\hat H(\Phi^{-1}(\Pm(x)), x)}}} \left( \pdata(\Phi^{-1}(\Pm(x))) + h(\sigma^2) \right)
  \end{align}
  where $\abs{h(\sigma^2)}$ is $O(\sigma)$.
  Now we take logarithmic and use the fact that $\log(A+B) = \log(A) + \log(1+B/A)$,
  we obtain
  \begin{align*}
    &\fakeeq \log p_{\sigma}(x) \\
    &= -\frac{\norm{ x - \Pm(x) }^2}{2\sigma^2} + \frac{n-d}{2} \log(2\pi \sigma^2) +
    \log \abs*{\hat H(\Phi^{-1}(\Pm(x)), x)}^{-1/2} + \\
    &\fakeeq \log \left( \pdata(\Phi^{-1}(\Pm(x))) + h(\sigma^2) \right) \\
    &= -\frac{\norm{ x - \Pm(x) }^2}{2\sigma^2} + \frac{n-d}{2} \log(2\pi \sigma^2) +
    \log \pdata(\Phi^{-1}(\Pm(x))) +  \\
    &\fakeeq \log \abs*{\hat H(\Phi^{-1}(\Pm(x)), x)}^{-1/2} +
    \log \left( 1 + \frac{h(\sigma^2)}{\pdata(\Phi^{-1}(\Pm(x)))} \right).
  \end{align*}
  Therefore, we have
  \begin{align*}
    \hat p(x, \sigma) = \log \left( 1 + \frac{h(\sigma^2)}{\pdata(\Phi^{-1}(\Pm(x)))} \right)
  \end{align*}
  The remaining is to show that $h(\sigma^2) / \pdata(\Phi^{-1}(\Pm(x)))$
  is uniformly $o(1)$ for all $x \in T_{\cM}(\epsilon)$.
  Since the manifold is compact, $\pdata(u)$ is uniformly bounded away from zero (see \Cref{remark:boundedness_in_compact_manifold}).
  The remaining is to find a suitable $\Omega'$ and upper and lower bound the constants
  in \Cref{corollary:laplace}. We will discuss this later.
  
  Now let us look at the case of VP.
  The only difference is in the exponential, we changed from $\norm{x - \Phi(u)}^2$
  to
  \[
  \norm{x - \sqrt{1-\sigma^2} \Phi(u)}^2 = \norm{x - \Phi(u) + \left(1-\sqrt{1-\sigma^2}\right) \Phi(u)}^2.
  \]
  If we do a Taylor expansion of $1 - \sqrt{1-\sigma^2}$:
  \begin{align*}
    1 - \sqrt{1-\sigma^2} = \frac{1}{2} \sigma^2 + o(\sigma^2).
  \end{align*}
  Using this expansion, we have that
  \begin{align*}
   &\fakeeq \norm{x - \Phi(u) + \left(1-\sqrt{1-\sigma^2}\right) \Phi(u)}^2 \\
   & = \norm{x - \Phi(u)}^2 + \sigma^2 \inp*{\Phi(u)}{x - \Phi(u)}
     + o(\sigma^2) \inp*{x}{\Phi(u)}.
  \end{align*}
  Then we can use the same argument as in the proof \Cref{lemma:limit_f_theta}
  to show that the $o(\sigma^2)$ does not affect the approximation.
  Specifically, let
  \begin{align*}
    J \coloneqq \exp\left(-\frac{\norm{x - \Pm(x)}^2}{2\sigma^2}\right) \frac{\left(2\pi \sigma^2\right)^{(n-d)/2}}{\sqrt{\abs*{\hat H(\Phi^{-1}(\Pm(x)), x)}}},
  \end{align*}
  and
  \begin{align*}
    K \coloneqq \frac{1}{(2\pi \sigma^2)^{d/2}} \exp\left(-\frac{\norm{x - \Phi(u)}^2}{2\sigma^2}\right) \exp\left(- \frac{1}{2} \inp*{\Phi(u)}{x - \Phi(u)} \right)\pdata(u).
  \end{align*}
  We have
  \begin{align*}
    &\fakeeq \int_{\cM} \frac{1}{(2\pi \sigma^2)^{d/2}} \exp\left(-\frac{\norm{x - \sqrt{1-\sigma^2} \Phi(u)}^2}{2\sigma^2}\right) \pdata(u) du \\
    &= \int_{\cM} K du + 
    \int_{\cM} K \left(\exp\left(o(1) \inp*{\Phi(u)}{x}\right) - 1\right)du \\
    &\leq \int_{\cM} K du + \int_{\cM} K o(1) du \\
    &\leq J \left( \pdata(\Phi^{-1}(\Pm(x))) \exp\left(-\frac{1}{2} \inp*{\Pm(x)}{x - \Pm(x)}\right)
     + o(1) \right).
  \end{align*}
  The rest of the proof follows similarly to the proof of the VE case.

  Then we need to discuss the upper and lower bounds in \Cref{corollary:laplace}.
  For the upper bounds, since the manifold is compact, there exists such uniform
  upper bounds for $\{c_i\}_1^6$ (see \Cref{remark:boundedness_in_compact_manifold}).
  For the lower bounds we first consider $\lambda_{\min}\left(\hat H(u, x)\right)$.
  The part $\frac{\partial \Phi(u)}{\partial u}^\T
  \frac{\partial \Phi(u)}{\partial u}$ is positive definite and uniformly
  bounded away from zero for all $u$. 
  The eigenvalues of other part, i.e., $\inp*{\frac{\partial \Phi(u)}{\partial u_i \partial u_j}}{\Phi(u) - x}$,
  may be negative.
  However, as long as its eigenvalues are small enough, by Weyl's inequality,
  we can still lower bound the smallest eigenvalue of $\hat H(u, x)$.
  The eigenvalues of $\inp*{\frac{\partial \Phi(u)}{\partial u_i \partial u_j}}{\Phi(u) - x}$,
  can then be bounded by $\norm{\nabla^2 \Phi(u)}$ $\norm{\Phi(u) - x}$.
  Therefore, as long as the tubular neighborhood and the set $\Omega'$ is small
  enough, we can lower bound $\lambda_{\min}\left(\hat H(u, x)\right)$.
  Formally, let $G > 0$ be the lower bound of the smallest eigenvalue of
  $\frac{\partial \Phi(u)}{\partial u}^\T
  \frac{\partial \Phi(u)}{\partial u}$.
  Let $C_2$ be the uniform upper bound of $\norm{\nabla^2 \Phi(u)}$,
  and $C_1$ be that of $\norm{\nabla \Phi(u)}$.
  Those constants are uniform for a fixed finite atlas since the manifold is compact.
  Let the radius of $\Omega'$ be $r_0$.
  We have that in $\Omega'$, $\lambda_{\min}\left(\hat H(u, x)\right) \geq G - C_2 (\norm{\Phi(u) - \Pm(x))}
  + \norm{\Pm(x) - x}) \geq G - C_2 (C_1 r_0 + \epsilon)$.
  Therefore, we can choose $r_0$ and $\epsilon$ small enough (but away from zero) such that
  $\lambda_{\min}\left(\hat H(u, x)\right) \geq G/2$, e.g., $\epsilon$ is the minimum
  of $G / (4 C_2)$ and the original $\epsilon$ in the tubular neighborhood definition,
  and $r_0 = G / (4 C_1 C_2)$. This way, $m_1$ can be lower bounded by $G r_0^2/2$.
\end{proof}

\subsection{Proofs for \Cref{sec:scale}}
\label{subsec:proof_scale}

The results in \Cref{subsec:general_proof,subsec:proof_scale} consider only
points in $T_{\cM}(\epsilon)$. Therefore, to use the results, we need first show that
the density outside the tubular neighborhood becomes negligible as $\sigma \to 0$.
In the following two lemmas, we will show the concentration of the density
for $p_{\sigma}$ and $\exp(-f_{\sigma})$.

\begin{lemma}\label[lemma]{lemma:limit_p_outside}
  Assume \Cref{assumption:manifold,assumption:pdata} holds.
  We have that $\lim_{\sigma \to 0} \int_{x \in T_{\cM}(\epsilon)} p_{\sigma}(x) dx = 1$.
\end{lemma}
\begin{proof}
  We have that
  \begin{align*}
    &\fakeeq \int_{x \in \bR^d / T_{\cM}(\epsilon)} p_{\sigma}(x) dx \\
    &= \int_{x \in \bR^d / T_{\cM}(\epsilon)} \int_{u \in \cM} \frac{1}{(2\pi \sigma^2)^{d/2}} \exp\left(-\frac{\norm{x - \Phi(u)}^2}{2\sigma^2}\right) \pdata(u) du dx \\
    &= \int_{u \in \cM} \pdata(u) \int_{x \in \bR^d / T_{\cM}(\epsilon)} \frac{1}{(2\pi \sigma^2)^{d/2}} \exp\left(-\frac{\norm{x - \Phi(u)}^2}{2\sigma^2}\right) dx du \\
    &\leq \int_{u \in \cM} \pdata(u) \int_{\norm{x - \Phi(u)} \geq \epsilon} \frac{1}{(2\pi \sigma^2)^{d/2}} \exp\left(-\frac{\norm{x - \Phi(u)}^2}{2\sigma^2}\right) dx du,
  \end{align*}
  where the exchange of the integral is justified by Tonelli's theorem with
  the non-negativity of the integrand.
  The last inequality holds since any point in $\bR^d / T_{\cM}(\epsilon)$
  is at least $\epsilon$ away from any point on the manifold.
  Now the inner integral is the integral of a Gaussian density with distance
  to the origin at least $\epsilon$. It will decay exponentially fast as $\sigma \to 0$.
  Let $Z$ be a standard Gaussian random variable of dimension $d$, and then the
  above integral is equivalent to
  \begin{align*}
    \int_{u \in \cM} \pdata(u) P\left(\norm{Z} \geq \frac{\epsilon}{\sigma}\right) du
    = P\left(\norm{Z} \geq \frac{\epsilon}{\sigma}\right).
  \end{align*}
  The RHS can be shown to decay exponentially fast by the Gaussian concentrations.
\end{proof}

\begin{lemma}\label[lemma]{lemma:limit_f_outside}
  Assume \Cref{assumption:manifold,assumption:pdata} holds.
  Further assume that
  \begin{align*}
    \sup_{x \in K} \norm{\nabla f_{\sigma}(x) + \nabla \log p_{\sigma}(x)} = o\left(\sigma^{-2}\right)
  \end{align*}
  We have that
  \begin{align*}
    \lim_{\sigma \to 0} \int_{x \in K \backslash T_{\cM}(\epsilon)} \exp(-f_{\sigma}(x)) dx = 0.
  \end{align*}
\end{lemma}
\begin{proof}
  For $x \notin T_{\cM}(\epsilon)$, the points are at least $\epsilon$ away
  from the manifold. Therefore, we have that
  \begin{align*}
    p_{\sigma}(x) \leq \int_{u \in \cM} \frac{1}{(2\pi \sigma^2)^{d/2}} \exp\left(-\frac{\epsilon^2}{2\sigma^2}\right) \pdata(u) du
    = \frac{1}{(2\pi \sigma^2)^{d/2}} \exp\left(-\frac{\epsilon^2}{2\sigma^2}\right),
  \end{align*}
  as $\pdata$ is a density function.
  Therefore, we have that
  \begin{align*}
    \exp(-f_{\sigma}(x)) \leq \frac{1}{(2\pi \sigma^2)^{d/2}} \exp\left(-\frac{\epsilon^2}{2\sigma^2} + o\left(\sigma^{-2}\right)\right),
  \end{align*}
  There exists $\sigma_0$, such that for all $\sigma < \sigma_0$,
  the $o(\sigma^{-2})$ term is upper bounded by $\epsilon^2 / 4 \sigma^2$.
  Then we have that
  \begin{align*}
    \int_{x \in K \backslash T_{\cM}(\epsilon)} \exp(-f_{\sigma}(x)) dx
    \leq \vol(K) \frac{1}{(2\pi \sigma^2)^{d/2}} \exp\left(-\frac{\epsilon^2}{4\sigma^2}\right).
  \end{align*}
  The RHS goes to zero as $\sigma \to 0$ as $\pdata$ is bounded.
\end{proof}

Now we are ready to prove our main theorems.

\begin{proof}[Proof of \Cref{thm:recover_pdata}]

First, since both $f_{\sigma}$ and $\log p_{\sigma}$ are $C^1$ functions
on $K$, we have the that $L^{\infty}$ norm of their gradients
is the same as the supremum.
First we will show that for any $\eta \geq -2$,
\begin{align*}
  \sup_{x \in K} \norm{\nabla f_{\sigma}(x) + \nabla \log p_{\sigma}(x)} = o(\sigma^{\eta})
  \quad \text{as } \sigma \to 0,
\end{align*}
implies that
\begin{align*}
  \sup_{x \in K} \abs{f_{\sigma}(x) + \log p_{\sigma}(x)} = o(\sigma^{\eta})
  \quad \text{as } \sigma \to 0.
\end{align*}
Given our assumption, for any two points $x, y \in K$, there exists a finite
length path, say $\gamma_{x, y}(\cdot): [0, 1] \to K$ with and $\norm{\gamma'}$
being upper bounded uniformly.
Consider an arbitrary point $x_0 \in K$, then we have
\begin{align*}
  &\fakeeq \Delta_{\sigma}(x) \coloneqq - f_{\sigma}(x) - \log p_{\sigma}(x) \\
  &= - f_{\sigma}(x_0) - \log p_{\sigma}(x_0) + \int_0^1 (- \nabla f_{\sigma}(\gamma(t)) - \nabla \log p_{\sigma}(\gamma(t))) \cdot \gamma'(t) dt \\
  &= \Delta_{\sigma} (x_0) + g(x, \sigma),
\end{align*}
where $\sup_x \abs{g(x, \sigma)}$ is $o(\sigma^{\eta})$ uniformly for $x \in K$
according to the assumption.
Further, we have that
\begin{align*}
  \int_{x \in K} \exp(-f_{\sigma}(x)) dx
  &= \int_{x \in K}  p_{\sigma}(x) \exp(\Delta_{\sigma}(x)) dx \\
  &= \int_{x \in K}  p_{\sigma}(x) \exp(\Delta_{\sigma}(x_0) + g(x, \sigma)) dx,
\end{align*}
which then imply that
\begin{align*}
  \Delta_{\sigma}(x_0) \geq
  \log \int_{x \in K} \exp(-f_{\sigma}(x)) dx
  - \log \int_{x \in K}  p_{\sigma}(x) dx
  - \sup_{x \in K} \abs{g(x, \sigma)},
\end{align*}
and
\begin{align*}
  \Delta_{\sigma}(x_0) \leq
  \log \int_{x \in K} \exp(-f_{\sigma}(x)) dx
  - \log \int_{x \in K}  p_{\sigma}(x) dx
  + \sup_{x \in K} \abs{g(x, \sigma)}.
\end{align*}
The first two terms on the right-hand side is $o(1)$ as $\sigma \to 0$
as our assumption about $f_{\sigma}$ and $\int_{K} p_{\sigma}(x) dx \geq \int_{T_{\cM}(\epsilon)} p_{\sigma}(x) dx \to 1$
according to \Cref{lemma:limit_p_outside}.
Thus, $\abs{\Delta_{\sigma}(x_0)}$ is $o(\sigma^{\eta})$.
Therefore, $\abs{\Delta_{\sigma}(x)}$ is $o(\sigma^{\eta})$ uniformly for $x \in K$.
Further we can apply \Cref{lemma:limit_f_outside} to conclude that
the density of $\exp(-f_{\sigma})$ concentrates in $T_{\cM}(\epsilon)$ as $\sigma \to 0$.

Then, we can prove the first conclusion that the support is on the manifold.
By the expansion of $\log p_{\sigma}$ in
\Cref{lemma:limit_V}, we have that
\begin{align*}
  f_{\sigma}(x) &= \frac{1}{2 \sigma^2} \norm{ x - \Pm(x) }^2 + o\left(1/\sigma^2\right).
\end{align*}
Then we can apply \Cref{lemma:limit_support} with
$f_{\theta}(x) = \sigma^2 f_{\sigma}(x)$, $\theta = \sigma^2$
and $\eta = \delta^2 / 2$
to conclude the claim.

To prove that the limiting distribution is $\pdata$ on the manifold,
we have
\begin{align*}
  f_{\sigma}(x) &= \frac{1}{2 \sigma^2} \norm{ x - \Pm(x)}^2 - \log \pdata(\Phi^{-1}(\Pm(x))) + \\ 
    &\fakeeq \log \sqrt{\abs*{\hat H(\Phi^{-1}(\Pm(x)), x)}} 
    + \frac{d-n}{2} \log(2\pi \sigma^2)
    + o(1).
\end{align*}
Then we can apply \Cref{thm:dist_f_theta}.
Then the $f_0$ becomes the distance function (changed to local coordinates),
and $f_1$ is $- \log \pdata + \log\sqrt{\abs*{\hat H(u), \Phi(u,r))}}$,
In addition, we note that for $r = 0$, $\sqrt{\abs*{\hat H(u), \Phi(u,r))}} = d \cM(u) / du$,
and therefore, we recover $\pdata$.
The $(d-n) \log(2 \pi \sigma^2)$ term is simply a constant and does not affect the
result after normalization.
One can replace $f_{\sigma}$ with $f_{\sigma} + \frac{d-n}{2} \log(2 \pi \sigma^2)$
and then apply \Cref{thm:dist_f_theta}, and this does not change the distribution
after normalization.

What remains is to ensure \Cref{assumption:f} holds, especially the second
condition, i.e., to ensure that the Hessian of $\norm*{\Phi(u,r) - \Phi(u)}^2 /2$
w.r.t. $r$ is uniformly bounded away from zero.
We can write $\Phi(u, r)$ as $\Phi(u) + \cN(u) r$,
where $\cN(u)$ is the normal vector field on the manifold $\cM$ at point $\Phi(u)$~\citep{weyl1939volume}.
We have that
\begin{align*}
\frac{\partial}{\partial r} \frac{\norm*{\Phi(u,r) - \Phi(u)}^2}{2}
= \frac{\partial \Phi(u, r)}{\partial r}^\T \left(\Phi(u, r) - \Phi(u)\right)
= \cN(u)^\T \cN(u) r = r,
\end{align*}
since the columns of $\cN(u)$ are orthonormal.
Therefore, the Hessian of $\norm*{\Phi(u,r) - \Phi(u)}^2 /2$ w.r.t. $r$
is simply the identity matrix, which satisfies the assumption.

To construct a $s(\sigma, x)$ such that the limiting distribution is arbitrarily,
say $\hat \pi$, we let $s(\sigma, x)$ being the gradient of
\begin{align*}
- \frac{1}{2 \sigma^2} \norm{ x - \Pm(x) }^2 + \log \hat \pi(\Phi^{-1}(\Pm(x))) -
\log \sqrt{\abs*{\hat H(\Phi^{-1}(\Pm(x)), x)}} + o(1).
\end{align*}
The difference between $f_{\sigma}$ and $\log p_{\sigma}$ is then $\Omega (1)$.
\end{proof}

\subsection{Manifold WKB Analysis of the Stationary Distribution}
\label{subsec:wkb_approximation}

A key difference between our theorem in \Cref{sec:recover_uniform} and the
results in \Cref{sec:scale} is that, in the former, the density does not admit
an explicit form. When $s(x,\sigma)$ is a gradient field, a closed-form
expression for the density is readily available; however, this property is not
guaranteed for most parameterized models, such as neural networks. We therefore
resort to the WKB approximation to approximate the stationary distribution.
Similarly to \Cref{subsec:general_proof}, we first present a general framework and then
apply it to our specific setting. 
We will show that SDE with the following form admits a stationary distribution
of the form \Cref{eq:general_density}.
Interested readers may refer to \citet{bouchet2016generalisation,bonnemain2019mean}
for more details on WKB applied on Fokker-Planck equation.

We consider the following SDE:
\begin{align*} 
  dx_t = b_{\theta}(x_t) dt + \sqrt{2 \theta} dW_t,
  \quad \text{with} \quad b_{\theta}(x) = -\nabla f_0(x) - \theta \nabla f_1(x) + \hat b(x, \theta),
\end{align*}
or the following SDE with the same stationary distribution,
\begin{align} \label{eq:general_sde}
  dx_t = \frac{b_{\theta}(x_t)}{\theta} dt + \sqrt{2} dW_t.
\end{align}
We assume that $\hat b(x, \theta)$ is uniformly $o(\theta)$ in $T_{\cM}(\epsilon)$
as $\theta \to 0$.
Also, we have $\argmin f_0(x) = \cM$.
This framework is general enough to cover the cases of 
\Cref{thm:recover_uniform,thm:with_guidence}.
We will see later that in these two cases, the function $f_0$ is the distance
function to the manifold, and $\theta$ will be chosen differently in different
cases. We make the following assumptions about the SDE.

Let $\pi_{\theta}(x)$ be the stationary distribution of the SDE~\Cref{eq:general_sde}.
First we assume the WKB ansatz:
\begin{assumption}[Local WKB ansatz] \label{assumption:wkb}
  We assume that $\lim_{\theta \to 0} \int_{T_{\cM}(\epsilon)} \pi_{\theta}(x) dx = 1$,
  and that $\pi_{\theta}(x)$ admits
  a local WKB form within compact set $T_{\cM}(\epsilon)$:
\begin{align*}
  \pi_{\theta}(x) \propto \exp\left(- \frac{V(x)}{\theta}\right) c_{\theta}(x)
  \quad \text{with} \quad c_{\theta}(x) = c_0(x) + \hat c(x, \theta),
\end{align*}
where $c_0 \in C^2(T_{\cM}(\epsilon))$ is positive, and $c_{\theta} \to c_0$ in $C^2(T_{\cM}(\epsilon))$.
We further assume that $V \in C^3(T_{\cM}(\epsilon))$
admits a unique solution.
\end{assumption}
The normalization constant can be explicitly written as
\begin{align*}
  \int_{x \in T_{\cM}(\epsilon)} \pi_{\theta}(x) dx / \int_{x \in T_{\cM}(\epsilon)} \exp\left(- \frac{V(x)}{\theta}\right) c_{\theta}(x) dx,
\end{align*}
since we have for $x \in T_{\cM}(\epsilon)$,
\begin{align*}
  \pi_{\theta}(x) &= \pi_{\theta}(x) \cdot \mathbf{1}_{T_{\cM}(\epsilon)}(x) = \pi_{\theta}(x \mid x \in T_{\cM}(\epsilon)) \pi_{\theta}(T_{\cM}(\epsilon)) \\
  &= \frac{c_{\theta}(x)\exp\left(- \frac{V(x)}{\theta}\right)}{\int_{x \in T_{\cM}(\epsilon)} c_{\theta}(x)\exp\left(- \frac{V(x)}{\theta}\right) dx}
  \pi_{\theta}(T_{\cM}(\epsilon)).
\end{align*}

Our goal would be to solve for $V(x)$ and $c_0(x)$ with the Fokker-Planck equation.
Once solved, to study the limit of $\pi_{\theta}$, we can
use results in \Cref{subsec:general_proof} as
\begin{align*}
  \pi_{\theta}(x) \propto \exp\left(- \frac{V(x) - \theta \log c_0(x) + o(\theta)}{\theta}\right).
\end{align*}

\begin{theorem} \label{thm:wkb_solution}
  Consider the SDE described in \Cref{eq:general_sde}. Assume \Cref{assumption:wkb} holds.
  Then we have that
  \begin{align*}
    V(x) = f_0(x), \quad c_0(x) = C \exp(-f_1(x)),
  \end{align*}
  for some constant $C$.
\end{theorem}
\begin{proof}
By Fokker-Planck equation for the stationary distribution, we have that
\begin{align*}
  0 = \dive \left(- b_{\theta}(x) \pi_{\theta}(x) + \theta \frac{\partial \pi_{\theta}(x)}{\partial x} \right).
\end{align*}
By plugging in the WKB ansatz, we have that
\begin{equation} \label{eq:fp_expand}
\begin{gathered}
  - \dive (b_{\theta}) c_{\theta} 
  - \inp*{b_{\theta}}{\frac{\partial c_{\theta}}{\partial x} - \frac{1}{\theta} \frac{\partial V}{\partial x} c_{\theta}} 
  + \theta \Tr{\frac{\partial^2 c_{\theta}}{\partial x^2}} 
  - 2 \inp*{\frac{\partial c_{\theta}}{\partial x}}{\frac{\partial V}{\partial x}} \\
  - \Tr{\frac{\partial^2 V}{\partial x^2} } c_{\theta}
  + \frac{1}{\theta} \norm*{\frac{\partial V}{\partial x}}^2 c_{\theta} = 0.
\end{gathered}
\end{equation}
Next by the method of WKB, we will equate different orders of $\theta$ in
the above equation to solve for $V(x)$ and $c_0(x)$, starting from
the lowest order $\theta^{-1}$.
It is easier to show a function is constant, therefore, for $c_0$, we will
define $\tilde c_0(x) = \exp(f_1(x)) c_0(x)$, and try to show that it is constant.

\paragraph{Order $\theta^{-1}$} In this order, we have that
\begin{align*}
  \inp*{\frac{\partial f_0}{\partial x}}{\frac{\partial V}{\partial x}} = \norm*{\frac{\partial V}{\partial x}}^2.
\end{align*}
This corresponds to the Hamilton-Jacobi equation typically appears in the WKB
approximation.
The equation gives the solution for $V(x)$ as $V(x) = f_0(x)$.
Plugging this solution into \Cref{eq:fp_expand}, we can get
\begin{equation*}
\begin{gathered}
  - \Tr{- \theta \frac{\partial^2 f_1}{\partial x^2} + \frac{\partial \hat b}{\partial x}} c_{\theta} 
  - \inp*{b_{\theta}}{\frac{\partial c_{\theta}}{\partial x}}
   + \inp*{-\theta \frac{\partial f_1}{\partial x} + \hat b}{\frac{1}{\theta} \frac{\partial f_0}{\partial x} c_{\theta}} 
  + \theta \Tr{\frac{\partial^2 c_{\theta}}{\partial x^2}} \\
  - 2 \inp*{\frac{\partial c_{\theta}}{\partial x}}{\frac{\partial f_0}{\partial x}} = 0.
\end{gathered}
\end{equation*}
We will work with this equation for equating the higher orders.

\paragraph{Order $\theta^0$} In this order, we have that
\begin{align*}
  \inp*{\frac{\partial f_1}{\partial x}}{\frac{\partial f_0}{\partial x}} c_0
  + \inp*{\frac{\partial c_0}{\partial x}}{\frac{\partial f_0}{\partial x}} = 0.
\end{align*}
This is known as the transport equation~\citep{bouchet2016generalisation}.
It shows how $c_0$ changes along the
gradient of $f_0$.
Next, we express the equation in terms of $\tilde c_0$:
\begin{align} \label{eq:O0}
  \inp*{\frac{\partial \tilde  c_0}{\partial x}}{\frac{\partial f_0}{\partial x}} = 0.
\end{align}
This implies that along the gradient of $f_0$, $\tilde c_0$ is constant.
Since the manifold $\cM$ consists of the minimizers of $f_0$, for any point $x$
in $K$, the value of $\tilde c_0(x)$ is the same as the value at the corresponding
minimizer $y$ on $\cM$ following the gradient flow of $f_0$. Formally, we have
\begin{align*}
  \tilde c_0(x) = \tilde c_0(\psi^x(+\infty)),
\end{align*}
where $\psi^x(t)$ follows $d\psi^x(t) / dt = -\nabla f_0(\psi^x(t))$ with $\psi^x(0) = x$
given the initial condition $\psi^x(0) = x$.
Therefore, we see that to solve for $\tilde c_0$, we need to know the value
of it on $\cM$.
We find that the next order equation will help us to solve for $\tilde c_0$ on $\cM$.

\paragraph{Order $\theta^1$} In this order, if we directly find all terms
in \Cref{eq:fp_expand} that are of order $\theta^1$, we will find that it 
includes higher order terms, e.g., $\hat c(x, \theta)$.
However, since we only care about the solution on $\cM$, we evaluate the
equation on $\cM$ and interestingly find that it does not include such higher order terms,
as crucially the factor $\partial f_0 / \partial x$ becomes 0 at $\cM$.
Specifically, for $x \in \cM$, we have that
\begin{align*}
  \Tr{\frac{\partial^2 f_1}{\partial x^2}} c_0
  + \inp*{\frac{\partial f_1}{\partial x}}{\frac{\partial c_0}{\partial x}}
  + \Tr{\frac{\partial^2 c_0}{\partial x^2}} = 0.
\end{align*}
Replacing $c_0$ with $\tilde c_0 \exp(-f_1)$, we have that
\begin{align*}
  \Tr{\frac{\partial^2 \tilde c_0}{\partial x^2}}
  - \inp*{\frac{\partial \tilde c_0}{\partial x}}{\frac{\partial f_1}{\partial x}} = 0.
\end{align*}
Our goal here would be to solve for $\tilde c_0$ on $\cM$, and apparently
it would be helpful to convert the equation to the local coordinates and
establish a PDE for the manifold chart coordinate $u$.

\paragraph{Local coordinates}
We convert the above order $\theta^1$ equation about $\tilde c_0$ 
to the local coordinates $z = (u, r)$ and get that for $r=0$,
i.e., points on $\cM$,
\begin{equation} \label{eq:local_O1}
\begin{aligned}
  0 &= \frac{1}{\abs{J}} \dive_z \left( \abs{J} G^{-1} \frac{\partial \tilde c_0}{\partial z}  \right)
  - \inp*{\frac{\partial \tilde c_0}{\partial z}}{G^{-1} \frac{\partial f_1}{\partial z}} \\
   &= \frac{1}{\abs{J}} \left( \inp*{\dive_z \left( \abs{J} G^{-1} \right)}{\frac{\partial \tilde c_0}{\partial z}}
   +  \Tr{\abs{J} G^{-1} \frac{\partial^2 \tilde c_0}{\partial z^2}} \right)
   - \inp*{\frac{\partial \tilde c_0}{\partial z}}{G^{-1} \frac{\partial f_1}{\partial z}},
\end{aligned}
\end{equation}
where $G = J^\T J$ and the divergence of a matrix is understood as the 
divergence of the column vectors.
Note that we cannot simply conclude from the above equation that $\tilde c_0$ is
constant, by say, the strong maximum principle, since the gradients of $\tilde c_0$
include not only the manifold chart coordinate $u$ but also the normal coordinate $r$.
Therefore, we have to further derive a PDE about $u$ and any gradients of
$\tilde c_0$ w.r.t. $r$ should be replaced by known functions.
Fortunately those gradients can be solved by the equation we obtain at order $\theta^0$.

First, let us derive from \Cref{eq:local_O1} a PDE about $u$:
\begin{lemma} \label{lemma:local_O1_u}
  From \Cref{eq:local_O1}, we have that for $r=0$,
  \begin{align} \label{eq:local_O1_u}
    \Delta_{\cM} \tilde c_0 (u)
    - \inp*{\frac{\partial \tilde c_0}{\partial u}}{g^{-1}\frac{\partial f_1}{\partial u}}
    + \frac{1}{\sqrt{\abs{g}}} \inp*{\frac{\partial \abs{J}}{\partial r}}{\frac{\partial \tilde c_0}{\partial r}}
    - \inp*{\frac{\partial \tilde c_0}{\partial r}}{\frac{\partial f_1}{\partial r}}
    + \Tr{\frac{\partial^2 \tilde c_0}{\partial r^2}} = 0,
  \end{align}
  where $\Delta_{\cM}$ is the Laplace-Beltrami operator on $\cM$.
\end{lemma}
\begin{proof}
  Let the index $i, j$ when showing at $\partial$ be derivatives w.r.t. the $i$ or $j$-th coordinate of $u$,
  and let $p, q$ be the derivatives w.r.t. $r$ respectively.
  Any index variable that is not explicitly defined is understood to be summed over.
  From \Cref{eq:local_O1}, by carefully expanding the divergence, the inner product term becomes
  \begin{align*}
    &\fakeeq \evalat{\inp*{\dive\left(\abs{J} G^{-1} \right)}{\nabla \tilde c_0}}{r=0} \\
    &= \sqrt{\abs{g}} \partial_j \left[g^{-1}\right]_{i,j} \partial_i \tilde c_0
    + \left[g^{-1}\right]_{i,j} \partial_j \sqrt{\abs{g}} \partial_i \tilde c_0
    - \sqrt{\abs{g}}\left[g^{-1}\right]_{i, k} \evalat{\partial_p d_{p,k}}{r=0} \partial_i \tilde c_0
  + \partial_p \abs{J} \partial_p \tilde c_0.
  \end{align*}
  For the trace term, we have
  \begin{align*}
  \evalat{\Tr{\abs{J} G^{-1} \nabla^2 \tilde c_0 }}{r=0} =
  \sqrt{\abs{g}} \left[g^{-1}\right]_{i,j} \partial_{j,i} \tilde c_0
  + \sqrt{\abs{g}} \partial_{p,p} \tilde c_0.
  \end{align*}
  Now we look at \Cref{eq:local_O1_u}.
  From the definition of Laplace-Beltrami operator, we have
  \begin{align*}
    \Delta_{\cM} \hat c_0(u) 
    &= \frac{1}{\sqrt{\abs{g}}} \partial_i \left(\sqrt{\abs{g}} \left[g^{-1}\right]_{i,j} \partial_j \tilde c_0\right) \\
    &= \frac{1}{\sqrt{\abs{g}}} \partial_i \sqrt{\abs{g}} \left[g^{-1}\right]_{i,j} \partial_j \tilde c_0
    + \partial_i \left[g^{-1}\right]_{i,j} \partial_j \tilde c_0
    + \left[g^{-1}\right]_{i,j} \partial_{i,j} \tilde c_0.
  \end{align*}
  Since $G^{-1}$ evaluated at $r=0$ is $\begin{bmatrix}
    g^{-1} & 0 \\
    0 & I
  \end{bmatrix}$,
  the term $- \inp*{\frac{\partial \tilde c_0}{\partial z}}{G^{-1} \frac{\partial f_1}{\partial z}}$
  in \Cref{eq:local_O1} matches
  $- \inp*{\frac{\partial \tilde c_0}{\partial u}}{g^{-1}\frac{\partial f_1}{\partial u}} - \inp*{\frac{\partial \tilde c_0}{\partial r}}{\frac{\partial f_1}{\partial r}}$
  in \Cref{eq:local_O1_u}.
  Now compare the terms of \Cref{eq:local_O1_u} and \Cref{eq:local_O1},
  the only remaining term is 
  \begin{align*}
  \left[g^{-1}\right]_{i, k} \evalat{\partial_p d_{p,k}}{r=0} \partial_i \tilde c_0,
  \end{align*}
  which we will prove is 0. We will show that $\evalat{\sum_p \partial_p d_{p,k}}{r=0} = 0$.

  Since the columns of $\cN$ are orthonormal, we have for any $p$, $\sum_i \left(\cN_{i,p}\right)^2 = 1$.
  Taking derivative for both sides to $u_j$, we have for any $p, j$, $\sum_i \cN_{i,p} \partial_j \cN_{i, p} = 0$.
  We also have by definition that for any $p, j$,
  \begin{align*}
  \left[\cN^\T \nabla \cN r \right]_{p, j} = \cN_{i,p} \partial_j \cN_{i,l} r_l.
  \end{align*}
  Using the above two results, we have for any $j$,
  \begin{align*}
    \sum_p \partial_p d_{p,j} = \sum_p \partial_p \left(\cN_{i,p} \partial_j \cN_{i, l} r_l\right)
  = \sum_p \cN_{i,p} \partial_j \cN_{i, p} = 0. 
  \end{align*}
\end{proof}

From \Cref{eq:local_O1_u}, we see that it contains gradients of $\tilde c_0$ w.r.t. $r$,
which we will solve by the order $\theta^0$ equation.

\begin{lemma} \label{lemma:grad_r}
  From \Cref{eq:O0}, we have that on the manifold $\cM$,
  \begin{align*}
    \frac{\partial \tilde c_0}{\partial r} (u, 0) = 0,
    \quad \text{and} \quad \Tr{\frac{\partial^2 \tilde c_0}{\partial r^2}} (u, 0)
    = \inp*{h(u)}{\frac{\partial \tilde c_0}{\partial u}(u, 0)},
  \end{align*}
  where $h(u)$ does not contain the unknown function $\tilde c_0$.
\end{lemma}
\begin{proof}
  Since we care about the evaluation of the equation on $\cM$, we start by
  changing the coordinates to the local coordinates $z = (u, r)$ from
  \Cref{eq:O0} to get that
  \begin{align*}
    \inp*{\frac{\partial \tilde c_0}{\partial z}}{G^{-1}\frac{\partial f_0}{\partial z}} = 0.
  \end{align*}
  Next, we compute the gradient w.r.t. $z$:
  \begin{align} \label{eq:O0_local_grad}
    \frac{\partial^2 \tilde c_0}{\partial z^2} G^{-1} \frac{\partial f_0}{\partial z}
     + \frac{\partial^2 f_0}{\partial z^2} G^{-1} \frac{\partial \tilde c_0}{\partial z}
     + \left(\frac{\partial \vecop{G^{-1}}}{\partial z}\right)^{T} 
     \left(\frac{\partial f_0}{\partial z} \otimes \frac{\partial \tilde c_0}{\partial z}\right) = 0,
  \end{align}
  where $\otimes$ is the Kronecker product.
  When we evaluate this equation at $r=0$, the factor $\partial f_0 / \partial r$
  becomes 0, $ G^{-1}(u, 0) = \begin{bmatrix}
    g^{-1} & 0 \\
    0 & I
  \end{bmatrix}$ and $\frac{\partial^2 f_0}{\partial z^2} (u, 0) = 
  \begin{bmatrix}
    0 & 0 \\
    0 & \partial^2 f_0 / \partial r^2 (u, 0)
  \end{bmatrix}. $
  Then we have
  \begin{align*}
    \frac{\partial^2 f_0}{\partial r^2} (u, 0) \frac{\partial \tilde c_0}{\partial r} (u, 0) = 0.
  \end{align*}
  Since $\frac{\partial^2 f_0}{\partial r^2} (u, 0)$ is full-rank, we have that
  $\frac{\partial \tilde c_0}{\partial r} (u, 0) = 0$.

  Next, we compute gradient again for \Cref{eq:O0_local_grad}, and evaluate at $r=0$.
  Ignoring $\partial f_0 / \partial z$ which is 0, we have the $i, j$-th element of the matrix is
  \begin{equation} \label{eq:O0_local_grad_2}
  \begin{gathered}
    \left[\frac{\partial^2 \tilde c_0}{\partial z^2} G^{-1} \frac{\partial^2 f_0}{\partial z^2}\right]_{i, j}
    + \left[ \frac{\partial^2 f_0}{\partial z^2} G^{-1} \frac{\partial^2 \tilde c_0}{\partial z^2}\right]_{i, j}
    + \frac{\partial^3 f_0}{\partial z_i \partial z_k \partial z_j} 
        \left[G^{-1} \frac{\partial \tilde c_0}{\partial z}\right]_k \\
    + \frac{\partial^2 f_0}{\partial z_i \partial z_k} 
      \frac{\partial G^{-1}_{k, p}}{\partial z_j} \frac{\partial \tilde c_0}{\partial z_p}
    + \frac{\partial \tilde c_0}{\partial z_k} \frac{\partial G^{-1}_{k, p}}{\partial z_i}
       \frac{\partial^2 f_0}{\partial z_p \partial z_j} = 0,
  \end{gathered}
  \end{equation}
  where $\partial \tilde c_0 / \partial r$ is 0.
  The first two terms have nice structure when evaluated at $r=0$, as
  \begin{align*}
    \frac{\partial^2 \tilde c_0}{\partial z^2} G^{-1} \frac{\partial^2 f_0}{\partial z^2}
     = \begin{bmatrix}
      0 & \frac{\partial^2 \tilde c_0}{\partial u \partial r} \frac{\partial^2 f_0}{\partial r^2} \\
      0 & \frac{\partial^2 \tilde c_0}{\partial r^2} \frac{\partial^2 f_0}{\partial r^2}
     \end{bmatrix}
     \quad \text{and} \quad
    \frac{\partial^2 f_0}{\partial z^2} G^{-1} \frac{\partial^2 \tilde c_0}{\partial z^2}
     = \begin{bmatrix}
      0 & 0 \\
      \frac{\partial^2 f_0}{\partial r^2} \frac{\partial^2 \tilde c_0}{\partial r \partial u}  &
      \frac{\partial^2 f_0}{\partial r^2} \frac{\partial^2 \tilde c_0}{\partial r^2}
     \end{bmatrix}.
  \end{align*}
  We then multiply \Cref{eq:O0_local_grad_2} by matrix $
  \begin{bmatrix}
  0 & 0 \\
  0 & \left(\frac{\partial^2 f_0}{\partial r^2}\right)^{-1}
  \end{bmatrix}
  $ from the left, and get
  \begin{align*}
    \begin{bmatrix}
      0 & 0 \\
      0 & \left(\frac{\partial^2 f_0}{\partial r^2}\right)^{-1} \frac{\partial^2 \tilde c_0}{\partial r^2} \frac{\partial^2 f_0}{\partial r^2}
    \end{bmatrix}
    +
    \begin{bmatrix}
      0 & 0 \\
      \frac{\partial^2 \tilde c_0}{\partial r \partial u}  & \frac{\partial^2 \tilde c_0}{\partial r^2}
    \end{bmatrix}
    + \text{remaining terms} = 0.
  \end{align*}
  Since $\partial \tilde c_0 / \partial r$ is 0, the element of the remaining 
  terms all have one and only one factor of $\partial \tilde c_0 / \partial u_i$
  for some $i$.
  Taking the trace of the above equation, and we have proved the second statement.
\end{proof}

Now we plug in \Cref{lemma:grad_r} to \Cref{lemma:local_O1_u}, and obtain
a PDE about $\tilde c_0(\cdot, 0)$ on $u$ whose second order derivatives are the Laplace-Beltrami operator,
and the zero-th order term, i.e., the term that includes the function value $\tilde c_0(\cdot, 0)$, is 0.
Therefore, we can conclude by strong maximum principle~\citep[Theorem~3.5]{gilbarg1977elliptic} that $\tilde c_0(\cdot, 0)$ is a constant.
According to the equation at order $\theta^0$, we obtain that
$\tilde c_0$ off-manifold is the same constant.
\end{proof}

\subsection{Proof for \Cref{sec:recover_uniform}}
\label{subsec:proof_recover_uniform}

We will first prove \Cref{thm:recover_uniform_gradient}, which follows
similar proof technique as \Cref{thm:recover_pdata}, and then turn to
the harder case of \Cref{thm:recover_uniform}.

\begin{proof}[Proof of \Cref{thm:recover_uniform_gradient}]
  The proof follows the same as \Cref{thm:recover_pdata}, except that now we use
  \Cref{thm:dist_f_theta} with $\theta = \sigma^{2-\alpha}$.
  In this case, $f_0(x) = \norm{x - \Pm(x)}^2/2 $,
  $f_1 \equiv 0$ and all other terms are asymptotically small compared
  to $\sigma^{2-\alpha}$.
  According to the proof of \Cref{thm:recover_pdata}, the determinant of the
  Hessian of $f_0$ in the normal direction is the same for all $u$,
  therefore, we recover the uniform distribution on the manifold.
  
  The only thing remains to verify is to ensure
  \begin{align*}
    \lim_{\sigma \to 0} \int_{\bR^d \backslash T_{\cM}(\epsilon)} \tilde \pi_{\sigma}(x) dx
    = \lim_{\sigma \to 0} \frac{\int_{\bR^d \backslash T_{\cM}(\epsilon)} \exp(-\sigma^{\alpha} f_{\sigma}(x)) dx}
    {\int_{\bR^d} \exp(-\sigma^{\alpha} f_{\sigma}(x)) dx} = 0.
  \end{align*}
  Since we have $\lim_{\sigma \to 0} \int_K \tilde \pi(x) dx \to 1$, we only
  need to consider within $K$.
  For the numerator, we can do similarly as \Cref{lemma:limit_f_outside} to obtain
  \begin{align*}
    \int_{K \backslash T_{\cM}(\epsilon)} \exp(-\sigma^{\alpha} f_{\sigma}(x)) dx
    \leq \vol(K) \left(\frac{1}{(2\pi \sigma^2)^{d/2}} \right)^{\sigma^{\alpha}} 
    \exp\left(-\frac{\epsilon^2}{4\sigma^{2-\alpha}} + o\left(\sigma^{\alpha+\beta}\right)\right),
  \end{align*}
  where $2-\alpha > 0$ and $\alpha + \beta > 0$.
  There exists $\sigma_0$, such that for all $\sigma < \sigma_0$,
  the $o(\sigma^{\alpha+\beta})$ term is upper bounded by $\epsilon^2 / 8 \sigma^{2-\alpha}$.
  Then we have the numerator upper bounded by
  \begin{align*}
    \vol(K) \left(\frac{1}{(2\pi \sigma^2)^{d/2}} \right)^{\sigma^{\alpha}} 
    \exp\left(-\frac{\epsilon^2}{8\sigma^{2-\alpha}}\right).
  \end{align*}
  For the denominator, it is lower bounded by
  \begin{align*}
    &\fakeeq \int_{T_{\cM}(\epsilon / 2)} \left(\frac{1}{(2\pi \sigma^2)^{d/2}} \right)^{\sigma^{\alpha}} 
    \exp\left(-\frac{\norm{x - \Phi(x)}^2}{2 \sigma^{2-\alpha}} + o\left(\sigma^{\alpha+\beta}\right) \right) dx \\
    &\geq \int_{T_{\cM}(\epsilon / 2)} \left(\frac{1}{(2\pi \sigma^2)^{d/2}} \right)^{\sigma^{\alpha}} 
    \exp\left(-\frac{\epsilon^2}{8 \sigma^{2-\alpha}} + o\left(\sigma^{\alpha+\beta}\right) \right) dx.
  \end{align*}
  There exists $\sigma_1$, such that for all $\sigma < \sigma_1$,
  the $o(\sigma^{\alpha+\beta})$ term is lower bounded by $\epsilon^2 / 16 \sigma^{2-\alpha}$.
  Then the denominator is lower bounded by
  \begin{align*}
    \vol(T_{\cM}(\epsilon / 2)) \left(\frac{1}{(2\pi \sigma^2)^{d/2}} \right)^{\sigma^{\alpha}} 
    \exp\left(-\frac{\epsilon^2}{16\sigma^{2-\alpha}}\right).
  \end{align*}
  Therefore, the ratio is upper bounded by
  \begin{align*}
    \frac{\vol(K)}{\vol(T_{\cM}(\epsilon / 2))} \exp\left(-\frac{\epsilon^2}{16\sigma^{2-\alpha}}\right),
  \end{align*}
  which goes to zero as $\sigma \to 0$.
\end{proof}

Next, for \Cref{thm:recover_uniform}, we use results in \Cref{subsec:wkb_approximation} to find
an approximate stationary distribution of the SDEs considered in
\Cref{sec:recover_uniform}, and then use results in \Cref{subsec:general_proof}
to prove the main theorem.

\begin{proof}[Proof of \Cref{thm:recover_uniform}]
  The SDE we consider can be also written as
  \begin{align*}
    dX_t = \frac{\sigma^2 s(X_t, \sigma)}{\sigma^{2-\alpha}} dt + \sqrt{2} dW_t,
  \end{align*}
  Therefore, we want to apply \Cref{thm:wkb_solution} with $\theta = \sigma^{2-\alpha}$
  and $b_{\theta} = \sigma^2 s(X_t, \sigma)$.
  We assert that under our assumption of \Cref{thm:recover_uniform}, we can
  write
  \begin{align*}
    b_{\theta}(x) = - \frac{\partial \norm{x - \Pm(x)}^2 / 2}{\partial x} + o\left(\sigma^{2-\alpha}\right),
  \end{align*}
  meaning that $f_0 = \norm{x - \Pm(x)}^2 / 2$ and $f_1 \equiv 0$.
  We will discuss the proof of this later.
  If we have the above, by \Cref{thm:wkb_solution}, the stationary distribution
   in $T_{\cM}(\epsilon)$ is given by
  \begin{align*}
    \pi_{\sigma}(x) \propto \exp\left(- \frac{\norm{x - \Pm(x))}^2/2}{\sigma^{2-\alpha}}
    + o(1)\right),
  \end{align*}
  where the error in the prefactor is equivalent to the error in the exponent.
  The remaining proof follows the same as \Cref{thm:recover_uniform_gradient}.
  
  It remains to prove the assertion about $b_{\theta}$.
  A sufficient condition is that
  \begin{align} \label{eq:nabla_log_p_expansion}
    \sup_{x \in T_{\cM}(\epsilon)} \norm*{\nabla \log p_{\sigma}(x) + \frac{1}{\sigma^2} \frac{\partial \norm{x - \Pm(x)}^2 / 2}{\partial x}} = O(1).
  \end{align}
  Because if \Cref{eq:nabla_log_p_expansion} holds, we have uniformly for any $x \in T_{\cM}(\epsilon)$,
  \begin{align*}
    &\fakeeq \norm*{b_{\theta}(x) + \frac{\partial \norm{x - \Pm(x)}^2 / 2}{\partial x}} \\
    &= \norm*{\sigma^2 s(x, \sigma) + \frac{\partial \norm{x - \Pm(x)}^2 / 2}{\partial x}} \\
    &= \norm*{\sigma^2 s(x, \sigma) - \sigma^2 \nabla \log p_{\sigma}(x) + \sigma^2 \nabla \log p_{\sigma}(x)
      + \frac{\partial \norm{x - \Pm(x)}^2 / 2}{\partial x}} \\
    &\leq \norm*{\sigma^2 s(x, \sigma) - \sigma^2 \nabla \log p_{\sigma}(x)} 
      + \norm*{\sigma^2 \nabla \log p_{\sigma}(x) + \frac{\partial \norm{x - \Pm(x)}^2 / 2}{\partial x}} \\
    &= o(\sigma^{2 + \beta}) + O(\sigma^2) \\
    &= o(\sigma^{2-\alpha}),
  \end{align*}
  where the last inequality holds because $\alpha > \max\{-\beta, 0\}$.
  In the theorem, we assumed $L^{\infty}\left(T_{\cM}(\epsilon)\right)$ norm,
  which is the same as $\sup_{x \in T_{\cM}(\epsilon)}$ since $s(x, \sigma)$
  and $\nabla \log p_{\sigma}(x)$ are continuous.
  
  Therefore, it remains to prove \Cref{eq:nabla_log_p_expansion}. 
  We will prove for the case of VE, and the case of VP holds with similar argument.
  The gradient of the distance function can be written as:
  \begin{align*}
    \frac{\partial \norm*{x - \Pm(x)}^2/2}{\partial x} = \left(I - \left(\frac{\partial \Pm(x)}{\partial x}\right)^{\T}\right) \left(x - \Pm(x)\right)
    = x - \Pm(x),
  \end{align*}
  where the last equality holds because $x - \Pm(x)$ is orthogonal to the manifold
  and the image of $\frac{\partial \Pm(x)}{\partial x}$ is in the tangent space of the manifold~\citep{leobacher2021existence}.
  Then note that
  \begin{align*}
    \nabla \log p_{\sigma}(x) =
    \frac{\nabla p_{\sigma}(x)}{p_{\sigma}(x)}
    = \frac{\int_{\cM} \cN(x; u, \sigma^2 I) \pdata(u) \frac{\Phi(u) - x}{\sigma^2}du}{\int_{\cM} \cN(x; u, \sigma^2 I) \pdata(u) du}.
  \end{align*}
  For the denominator, follow the same as in the proof of \Cref{lemma:limit_V}
  to obtain that
  \begin{align*}
    p_{\sigma}(x) = \exp\left(-\frac{\norm{x - \Pm(x)}^2}{2\sigma^2}\right)
    \frac{\left(2\pi \sigma^2\right)^{(n-d)/2} \pdata(\Phi^{-1}(\Pm(x))}{\sqrt{\abs*{\hat H(\Phi^{-1}(\Pm(x)), x)}}} \left( 1 + O(\sigma) \right),
  \end{align*}
  since \Cref{eq:p_sigma_ve} holds and $\pdata$ is uniformly bounded away from zero.
  We could do the same for the numerator, however, the $O(\sigma)$ error is not enough
  here. Intuitively, the numerator would be 
  \begin{align*}
    \exp\left(-\frac{\norm{x - \Pm(x)}^2}{2\sigma^2}\right)
    \frac{\left(2\pi \sigma^2\right)^{(n-d)/2} \pdata(\Phi^{-1}(\Pm(x))}{\sqrt{\abs*{\hat H(\Phi^{-1}(\Pm(x)), x)}}} 
    \left( \frac{\Pm(x) - x}{\sigma^2} + O(1/\sigma) \right).
  \end{align*}
  Apparently, the error term is not enough to prove \Cref{eq:nabla_log_p_expansion}.

  Therefore, we turn to stronger Laplace's method result that has an error term
  of $O\left(\sigma^2\right)$, i.e., the $h(\theta)$ term in \Cref{corollary:laplace}
  could be improved to $O(\theta)$ instead of $O(\sqrt{\theta})$.
  However, such result should have the cost of requiring the function $F$ (as the
  notation use in \Cref{corollary:laplace}) to be $C^4$ and $g$ to be $C^2$,
  a stronger condition\footnote{Weaker condition such as $C^{1,1}$ is also possible, see \citet[Theorem 2.4]{majerski2015simple}.}.
  Formally, we have that
  \begin{align*}
    &\fakeeq \sigma^2 \nabla \log p_{\sigma}(x) + (x - \Pm(x)) \\
    &= \frac{\int_{\cM} \cN(x; u, \sigma^2 I) \pdata(u) \left(\Phi(u) - \Pm(x)\right) du}{\int_{\cM} \cN(x; u, \sigma^2 I) \pdata(u) du},
  \end{align*}
  and we want to prove its $L^{\infty}\left(T_{\cM}(\epsilon)\right)$ norm is $O(1)$.
  For any $x \in T_{\cM}(\epsilon)$ and $v \in \{v \mid \norm{v} = 1\}$, we have that
  \begin{align*}
    &\fakeeq v^\T \left(\sigma^2 \nabla \log p_{\sigma}(x) + (x - \Pm(x))\right) \\
    &= \frac{ \int_{\cM} \cN(x; u, \sigma^2 I) \pdata(u) v^\T \left(\Phi(u) - \Pm(x)\right) du}{ \int_{\cM} \cN(x; u, \sigma^2 I) \pdata(u) du} \\
    &= \frac{ \int_{\cM} \cN(x; u, \sigma^2 I) \pdata(u) \left(v^\T \left(\Phi(u) - \Pm(x)\right) + 1\right) du}{ \int_{\cM} \cN(x; u, \sigma^2 I) \pdata(u) du}  - 1.
  \end{align*}
  The last step where we add 1 is a simple trick because the Laplace's method we 
  will use does not allow the prefactor to be 0 at the minimizer.
  Next, we multiply the numerator and denominator by $\exp\left(\frac{\norm{x - \Pm(x)}^2}{2\sigma^2}\right) \frac{\sqrt{\abs*{\hat H(\Phi^{-1}(\Pm(x)), x)}}}{\left(2\pi \sigma^2\right)^{(n-d)/2}} $,
  so that their limit does not diminish to 0.
  For the numerator, we apply \citet[Theorem 2.4]{majerski2015simple} with
  their $n=1/\sigma^2$, $t=u$, $f(u)=\norm*{x - \Phi(u)}^2/2$, $\alpha=2$
  ($f(u)$ is $C^4$ since $\Phi(u)$ is $C^4$), $B_{\delta}$ can be selected
  the same as in the proof of \Cref{lemma:limit_V}, $g(u) = \pdata(u)(v^\T \left(\Phi(u) - \Pm(x)\right) + 1)$
  (g(u) is $C^2$ since $\pdata(u)$ is $C^2$), and the minimizer is $\Phi^{-1}(\Pm(x))$.
  The upper boundedness of the constants can be easily verified by compactness
  and one can show that they are uniform for $x$ and $v$.
  Crucially, $g(\Phi^{-1}(\Pm(x))) = \pdata(\Phi^{-1}(\Pm(x)))$ is uniformly lower bounded.
  The lower boundedness of $\lambda_{\min}$ can be reasoned in the same way as in the proof of \Cref{lemma:limit_V}.
  Therefore, we have
  \begin{align*}
    v^\T \left(\sigma^2 \nabla \log p_{\sigma}(x) + (x - \Pm(x))\right)
    = \frac{1 + O(\sigma^2)}{1 + O(\sigma^2)} - 1
    = O(\sigma^2).
  \end{align*}
  Since the bound is uniformly for $x$ and $\norm{v} = 1$, we have that
  \begin{align*}
    &\fakeeq \sup_{x \in T_{\cM}(\epsilon)} \norm*{\sigma^2 \nabla \log p_{\sigma}(x) + (x - \Pm(x))} \\
    &\leq \sup_{x \in T_{\cM}(\epsilon)} \sup_{\norm{v} = 1} v^\T \left(\sigma^2 \nabla \log p_{\sigma}(x) + (x - \Pm(x))\right)
    = O(\sigma^2),
  \end{align*}
  which proves \Cref{eq:nabla_log_p_expansion}.

\end{proof}

\subsection{Proof for \Cref{sec:bayesian_inverse}}
\label{subsec:proof_with_guidence}

\begin{proof}[Proof of \Cref{thm:with_guidence}]
  The proof follows the same as \Cref{thm:recover_uniform}, except that now we
  have $f_1 = v$ when applying \Cref{thm:wkb_solution} and \Cref{thm:dist_f_theta}.
\end{proof}

\section{Experimental Details and Further Experiments}
\label{sec:exp_details}

\subsection{Numerical Simulations on Ellipse}

\paragraph{Loss function.}
In our experiments, we train the score network to predict
\[
  \hat s(x, \sigma) \coloneqq \sigma^2 s(x, \sigma),
\]
instead of $s(x,\sigma)$ directly. This formulation is more stable across noise
levels, since the leading term in the score expansion is of order $1/\sigma^2$,
making $\hat s(x,\sigma)$ an $O(1)$ target.
With this choice, the training objective becomes
\begin{align*}
&\fakeeq \frac{1}{2} \, \mathbb{E}_{u \sim \pdata}
\mathbb{E}_{x \sim \cN(\Phi(u), \sigma^2 I)} 
\left[ \sigma^2 \left\| 
s(x, \sigma) + \tfrac{x - \Phi(u)}{\sigma^2} 
\right\|^2 \right] \\
&= \frac{1}{2} \, \mathbb{E}_{u \sim \pdata}
\mathbb{E}_{x \sim \cN(\Phi(u), \sigma^2 I)} 
\left[ \tfrac{1}{\sigma^2} \left\| 
\hat s(x, \sigma) + x - \Phi(u) \right\|^2 \right].
\end{align*}
The score function $s$ is parameterized by a neural network consisting of four
transformer blocks, each with hidden dimension $128$.

\paragraph{Data and noise.}
Training data is generated from a von Mises distribution with parameter
$\kappa = 1$. The injected Gaussian noise variance $\sigma^2$ is sampled from
a range $\sigma \in [0.01, 50]$.

\paragraph{Optimization.}
We use AdamW with weight decay $1\times 10^{-4}$ and global gradient clipping at
norm $1.0$. The initial learning rate is $3\times 10^{-3}$, decayed
cosine-schedule over $4\times 10^{-4}$ steps down to $1\%$ of its initial value,
after which training continues with a constant learning rate of $4\times 10^{-4}$.
The batch size is set to $1024$.

\paragraph{Sampling.}
For sampling, we run Langevin dynamics
\[
  dx_t = \hat s(x_t, \sigma_{\min}) \, dt + \sqrt{2\sigma_{\min}^2}\, dW_t,
\]
with $\sigma_{\min} = 0.01$. This process has the same stationary distribution
as
\[
  dx_t = s(x_t, \sigma_{\min}) \, dt + \sqrt{2}\, dW_t.
\]
For the TS Langevin dynamics, the diffusion coefficient is
$\sqrt{2 \sigma_{\min}^{2-\alpha}}$ instead of $\sqrt{2 \sigma_{\min}^2}$.
We employ the Euler--Maruyama scheme with a step size of $0.1$,
running $10{,}000$ steps with $10{,}000$ runs.

\subsection{Image Generation with Diffusion Models}
\newcommand{\nco}{n_{\mathrm{corr.}}}

\paragraph{Algorithm details.}
We use a pre-trained Stable Diffusion~1.5 model with a DDPM sampler in a
predictor--corrector (PC) scheme. The pre-trained network provides a denoiser
$\epsilon(x,t,y)$, and the corresponding classifier-free guidance (CFG) score at
time $t$ is
\begin{align*}
  s_t(x,y)
  &= \underbrace{\nabla_x \log p_t(x)}_{\text{unconditional score}}
     \;+\;
     w\,\underbrace{\bigl(\nabla_x \log p_t(x\mid y)-\nabla_x \log p_t(x)\bigr)}_{\text{conditional increment}}
  \\
  &= - \frac{1}{\sigma_t}\!\left[
     \epsilon(x,t,\emptyset)\;+\;w\bigl(\epsilon(x,t,y)-\epsilon(x,t,\emptyset)\bigr)
  \right],
\end{align*}
where $y$ is the conditioning input (prompt embedding), $w$ is the guidance
scale, $\sigma_t=\sqrt{1-\bar{\alpha}_t}$, and $\bar{\alpha}_t$ is as in
\citet{ho2020denoising}. Our tempered-score framework applies to this PC
sampler by modifying only the unconditional component while leaving the guided
increment unchanged:
\begin{align*}
  \tilde{s}_t(x,y)
  \;=\; -\frac{1}{\sigma_t}\!\left[
     \sigma_t^{\alpha}\,\epsilon(x,t,\emptyset)\;+\;
     w\bigl(\epsilon(x,t,y)-\epsilon(x,t,\emptyset)\bigr)
  \right],
\end{align*}
which is consistent with \Cref{eq:modified_langevin}.
Let $\{t_i\}$ denote the discrete reverse-time schedule. After each DDPM
predictor update at level $t_i$, we perform $\nco$ \emph{corrector}
steps of Langevin dynamics with the tempered score:
\begin{align*}
  x_{k+1}
  \;=\;
  x_k \;+\; \delta_i\,\tilde{s}_{t_i}(x_k,y) \;+\; \sqrt{2\delta_i}\,\xi_k,
  \qquad
  \xi_k \sim \mathcal{N}(0,I),
\end{align*}
where the step size $\delta_i$ follows \citet[Algorithm~5]{songscore}. After the
entire reverse process, we apply an additional $\nco$ deterministic
projection steps using the unconditional score (no guidance, no noise) to
further project onto the data manifold:
\begin{align*}
  d x_\tau = \nabla \log p_{t_0}(x_\tau) d\tau.
\end{align*}
We use the same number of projection steps for both the original PC baseline
and our TS to ensure a fair comparison.

\paragraph{Hyperparameter setting.}
We adopt the default configuration of Stable Diffusion~1.5
(\url{https://huggingface.co/stable-diffusion-v1-5/stable-diffusion-v1-5}).
Unless otherwise noted, all results in
\Cref{sec:exp_diffusion} use guidance scale $w=7.5$ and $30$ inference steps.
For the best-results reported in \Cref{tab:best_result}, we perform a grid search
over the number of corrector steps in $\{5,10,15,20,30\}$ and
$\alpha \in \{0.1,0.5,1.0,1.5\}$. The original PC baseline is tuned over the
same numbers of corrector step for fairness. For CLIP evaluations, we
generate $512$ images per setting and downscale each to $256\times 256$ before
computing the scores.

\subsection{Controlled Experiment with Ground Truth Scores}

To empirically validate the rate separation results in \Cref{thm:recover_pdata,thm:recover_uniform_gradient}, we designed a controlled experiment using synthetic data where the manifold and ground truth scores are known analytically.

We consider the unit circle manifold $\mathcal{M} = \{x \in \mathbb{R}^2 \mid \|x\|=1\}$ with a Von Mises distribution $\pdata(\theta) \propto \exp(\kappa \cos(\theta - \theta_0))$, where we used $\kappa=4$ and $\theta_0 = \pi$. This setup allows us to compute the analytic ground truth score $s^*(x, \sigma)$. We then inject a deterministic error field $e(x)$ into the true score:
$$\hat{s}(x, \sigma) = s^*(x, \sigma) + e(x), \quad \text{with} \quad e(x) = -\nabla \left( \frac{1}{2} \left\| x - \begin{bmatrix} 1 \\ 0 \end{bmatrix} \right\|^4 \right).$$
The magnitude of this error term $e(x)$ is $O(1)$ with respect to $\sigma$.

We compare the performance of the standard reverse diffusion process against our proposed TS Langevin dynamics using this corrupted score $\hat{s}$.
As shown in \Cref{fig:inject_noise}, the standard reverse diffusion process using $\hat{s}$ produces samples that deviate significantly from the ground truth $\pdata$, confirming that $O(1)$ score errors are sufficient to corrupt distributional recovery, while the TS Langevin dynamics with $\alpha=1$ robustly recovers the uniform distribution on the circle.

\begin{figure}[t]
\centering
\begin{subfigure}[t]{0.45\textwidth}
  \centering
  \includegraphics[width=\textwidth]{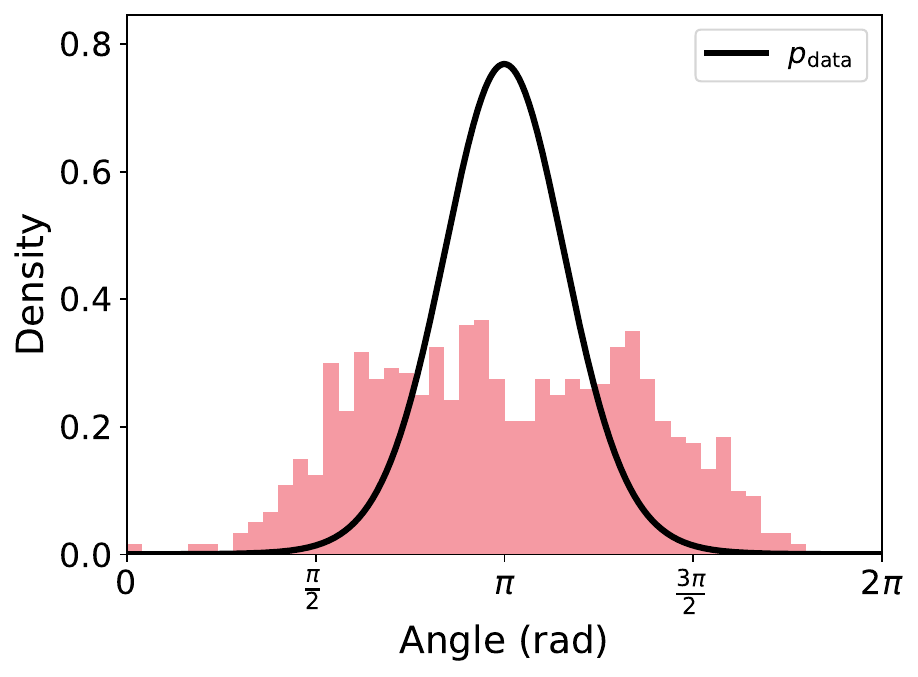}
  \caption{Distribution generated by Diffusion Model}
\end{subfigure}
\hfill
\begin{subfigure}[t]{0.45\textwidth}
  \centering
  \includegraphics[width=\textwidth]{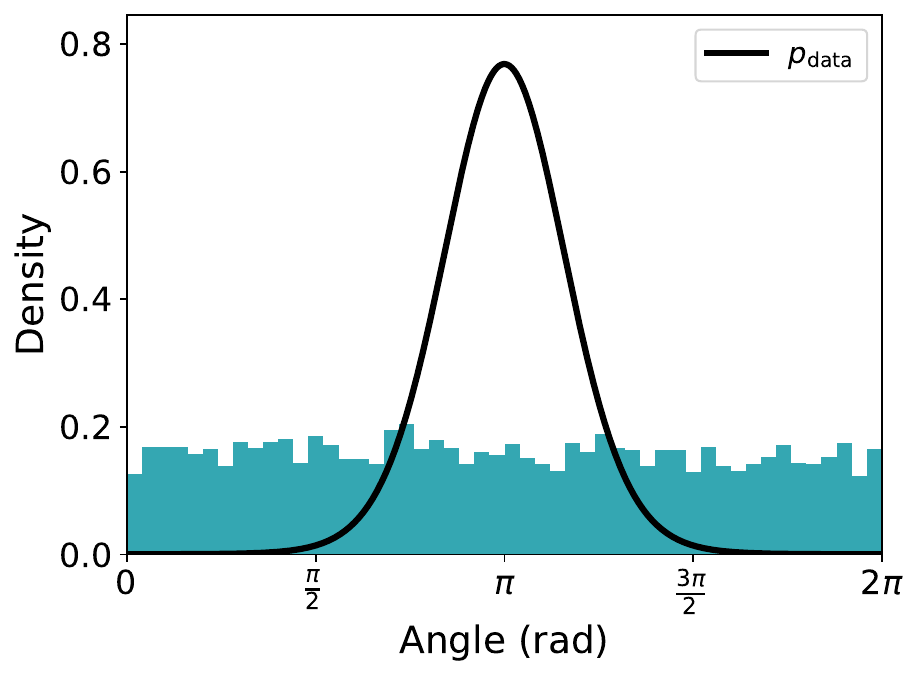}
  \caption{Distribution generated by TS-1}
\end{subfigure}
\caption{Comparison of distributions generated with VE diffusion model
versus our TS Langevin dynamics~\Cref{eq:modified_langevin} with $\alpha = 1$.}
\label{fig:inject_noise}
\end{figure}

\subsection{Sensitivity Analysis of Hyperparameter \texorpdfstring{$\alpha$}{α}}
\label{app:sensitivity_alpha}

To evaluate the sensitivity of the hyperparameter $\alpha$, we performed an ablation study using the Stable Diffusion 1.5 model, under the same setting as in \Cref{sec:exp_diffusion} of our paper. We tested $\alpha \in \{0, 0.1, 0.5, 1.0, 1.5\}$ across three prompt categories, with the number of corrector steps fixed at 10 and 20. Note that $\alpha=0$ corresponds to the standard predictor-corrector baseline.

\begin{table}[t]
  \centering
  \setlength{\tabcolsep}{2pt}
  \resizebox{0.95\textwidth}{!}{
  \begin{tabular}{l cc cc cc cc cc}
  \toprule
   & \multicolumn{2}{c}{\textbf{$\alpha = 0$}} & \multicolumn{2}{c}{\textbf{$\alpha = 0.1$}} & \multicolumn{2}{c}{\textbf{$\alpha = 0.5$}} & \multicolumn{2}{c}{\textbf{$\alpha = 1.0$}} & \multicolumn{2}{c}{\textbf{$\alpha = 1.5$}} \\
  \cmidrule(lr){2-3} \cmidrule(lr){4-5} \cmidrule(lr){6-7} \cmidrule(lr){8-9} \cmidrule(lr){10-11}
  \textbf{Prompt} & P-sim$\uparrow$ & I-sim$\downarrow$ & P-sim & I-sim & P-sim & I-sim & P-sim & I-sim & P-sim & I-sim \\
  \midrule
  \textbf{Architecture} & 27.13 & 81.81 & 27.12 & 81.73 & 27.14 & 81.67 & 27.27 & 81.57 & \textbf{27.32} & \textbf{81.52} \\
  \textbf{Furniture} & 29.30 & 81.24 & 29.32 & 81.37 & 29.33 & 81.06 & 29.58 & 80.95 & \textbf{30.16} & \textbf{80.76} \\
  \textbf{Car} & 26.30 & 87.57 & 26.30 & 87.58 & 26.31 & 87.44 & 26.37 & 87.42 & \textbf{26.50} & \textbf{87.34} \\
  \bottomrule
  \end{tabular}
  }
  \caption{Ablation of $\alpha$ for 10 corrector steps.}
  \label{tab:ablation_10}
\end{table}

\begin{table}[t]
  \centering
  \setlength{\tabcolsep}{2pt}
  \resizebox{0.95\textwidth}{!}{
  \begin{tabular}{l cc cc cc cc cc}
  \toprule
   & \multicolumn{2}{c}{\textbf{$\alpha = 0$}} & \multicolumn{2}{c}{\textbf{$\alpha = 0.1$}} & \multicolumn{2}{c}{\textbf{$\alpha = 0.5$}} & \multicolumn{2}{c}{\textbf{$\alpha = 1.0$}} & \multicolumn{2}{c}{\textbf{$\alpha = 1.5$}} \\
  \cmidrule(lr){2-3} \cmidrule(lr){4-5} \cmidrule(lr){6-7} \cmidrule(lr){8-9} \cmidrule(lr){10-11}
  \textbf{Prompt} & P-sim$\uparrow$ & I-sim$\downarrow$ & P-sim & I-sim & P-sim & I-sim & P-sim & I-sim & P-sim & I-sim \\
  \midrule
  \textbf{Architecture} & 26.87 & 81.60 & 26.85 & 81.56 & 26.97 & 81.49 & 27.06 & \textbf{80.97} & \textbf{27.10} & 81.13 \\
  \textbf{Furniture} & 28.98 & 81.72 & 28.99 & 81.65 & 29.07 & 81.40 & 29.52 & \textbf{81.15} & \textbf{30.20} & 81.39 \\
  \textbf{Car} & 26.26 & 88.06 & 26.26 & 88.09 & 26.25 & 87.95 & 26.28 & 88.07 & \textbf{26.62} & \textbf{87.70} \\
  \bottomrule
  \end{tabular}
  }
  \caption{Ablation of $\alpha$ 20 corrector steps.}
  \label{tab:ablation_20}
\end{table}

As shown in \Cref{tab:ablation_10,tab:ablation_20}, our method yields consistent improvements over the baseline ($\alpha=0$) once $\alpha$ is sufficiently large ($\alpha \ge 0.5$), demonstrating that the performance gains are robust and not limited to a narrow hyperparameter setting. The performance is particularly stable for $\alpha \in [1.0, 1.5]$, which aligns well with our theoretical framework (\Cref{thm:recover_uniform_gradient,thm:recover_uniform}) that guarantees convergence to the uniform distribution for any $\alpha < 2$. While we utilized $\alpha=1$ in \Cref{tab:corrector_steps} for simplicity, these results suggest that slightly more aggressive tempering ($\alpha=1.5$) can provide further gains in diversity and quality.

\section{Convergence of TS Langevin}
\label{sec:ts_convergence}
In this section, we deduce the mixing time analysis, i.e. the convergence analysis for a stochastic process, of the TS Langevin to the estimation of the Poincar\'e constant.
The goal is to show that TS Langevin is not necessarily slower—and can in fact be significantly faster—than the standard Langevin dynamics in terms of mixing time.
To carry out such an analysis, we assume that the score network is a gradient field, i.e. $s(\cdot, \sigma) = \nabla \log p_\theta$ for some parameterized density function.
WLOG, we assume $p_\theta$ is normalized as the normalizing factor does not affect the velocity field $s$.
\subsection{Convergence analysis of Langevin dynamics using functional inequality}
To analyze the convergence of Langevin dynamics, it is customary to use a functional inequality satisfied by the invariant measure $p_\infty$ of the Langevin dynamics (Here, $p_t$ denotes the density of the process at time $t$, and $p_\infty$ is its stationary distribution. This notation differs from $p_\theta$, and the distinction should be clear from context). In this response, we focus on the Poincaré inequality (PI): We say $p_\infty$ satisfies $\mathrm{PI}(C_\mathrm{PI})$ if for all $f \in H^1(p_\infty)$ (Sobolev space weighted by $p_\infty$),
$$
	\int \bigl(f - \int f d p_\infty\bigr)^2  d p_\infty \leq \frac{1}{C_{\mathrm{PI}}}  \int \bigl|\nabla f\bigr|^2  d p_\infty,
$$
where we call $C_{\mathrm{PI}} > 0$ is the Poincaré constant.

Consider the overdamped Langevin dynamics with potential $U_\sigma:\mathbb{R}^d\rightarrow \mathbb{R}$:
$$
	dX(t) = -\nabla U_{\sigma}(X(t))dt + \sqrt{2}dW(t),
$$
and let $p_t = \mathrm{Law}(X(t))$.
Under mild assumptions, $p_\infty \propto \exp(-U_\sigma)$ is the unique invariance measure of the above dynamics.
If $p_\infty \propto \exp(-U_\sigma)$ satisfies $\mathrm{PI}(C_{\mathrm{PI}}$), then
$$
	\chi^2(p_t, p_\infty) \le e^{-C_{\mathrm{PI}} t}\chi^2(p_0, p_\infty),
$$
where $\chi^2$ denotes the $\chi^2$-divergence. In particular, to ensure $\chi^2(p_t, p_\infty) \le \eta$ for some target accuracy $\eta > 0$, it suffices to take $t = O(\frac{1}{C_{\mathrm{PI}}}  \log \frac{1}{\eta})$. \emph{Thus, the larger the Poincaré constant, the faster the convergence.}

\subsection{Analyzing the effect of drift scaling to the Poincar\'e constant.}
Under the assumptions of our paper, the comparison between the mixing of standard Langevin and TS Langevin therefore reduces to comparing their Poincaré constants. We illustrate how drift scaling affects the Poincar\'e constant in the simple case where the data manifold is the unit circle:
$$
	\mathcal{M} = \{x \in \mathbb{R}^d : \|x\| = 1\}.
$$
In this case, the squared distance function can be computed in a closed form:
$$
	d(x) = \frac{1}{2}\mathrm{dist}^2(x, \mathcal{M}) = \frac{1}{2} \|x - \frac{x}{\|x\|}\|^2 = \frac{1}{2}(\|x\|-1)^2.
$$
Following section 5 of our paper, we assume the score error is $O(\sigma^{\beta})$ for some $-2 < \beta < 0$.
Recall that we assume the learned score is a gradient field, i.e. $s(\cdot, \sigma) = \nabla \log p_\theta$. Let us further suppose that the problem dimension is $d=2$, i.e. $x\in\mathbb{R}^2$, and the density function $p_\theta$ (corresponding to the learned score $s(\cdot, \sigma)$) has the following form
$$
	-\log p_\theta = \frac{1}{\sigma^2} d(x) + \sigma^{\beta} \phi(x), \mathrm{where}\ \phi(x) = (|x_1| - 1)^2,
$$
where $x_1$ denotes the first coordinate of $x$. Clearly, this function satisfies all requirement in our paper.
Crucially, such a construction ensures that the score error is $O(\sigma^{\beta})$.

\paragraph{Standard Langevin dynamics.}
We restate the standard Langevin dynamics for the ease of reference:
$$
	dX(t) = \nabla \log p_\theta (X(t)) dt + \sqrt{2} dW(t).
$$
Without temperature scaling, the error function $\phi(x)$ introduces two separated modes $(-1, 0)$ and $(+1, 0)$. For such a multimodal measure, classical Eyring-Kramers law or the large deviation principle results imply that the Poincaré constant can scale as
$$
	C^\mathrm{LD}_{\mathrm{PI}} = O(\exp\bigl(-\sigma^{\beta}\bigr)).
$$
Consequently, the mixing time of the original Langevin dynamics can become \emph{exponentially large} as $\sigma \to 0$.

\paragraph{TS Langevin.}
We restate the standard Langevin dynamics for the ease of reference:
$$
	dX(t) = \sigma^{\alpha}\nabla \log p_\theta (X(t)) dt + \sqrt{2} dW(t) = \nabla \log p^{\sigma^{\alpha}}_\theta (X(t)) dt + \sqrt{2} dW(t).
$$

Under mild conditions, the unique equilibrium measure is $p_\theta^{\sigma^\alpha}$.
We show that, under our standing assumptions and $\alpha > -\beta$, that its Poincaré constant,
denoted as $C_{\mathrm{PI}}^{\mathrm{TS}}$, is \emph{uniformly bounded away from zero}, \emph{independent of $\sigma$} for sufficiently small $\sigma$. Here we summarize the main steps:
\begin{itemize}
	\item Recall the Holley–Stroock perturbation principle \citep{holley1987logarithmic}:
	      Let $U$ and $\tilde U$ be two potential functions defined on $\mathbb{R}^d$.
	      Suppose that the corresponding Gibbs measures $p_\infty \propto \exp(-U)$ and $\tilde p_\infty \propto \exp(-\tilde U)$ satisfy Poincaré inequality with constants $C_\mathrm{PI}$ and $\tilde C_\mathrm{PI}$ respectively. One has
	      $$
		      \tilde C_{\mathrm{PI}} \geq \exp(-osc(\tilde U, U))C_{\mathrm{PI}},
	      $$
	      where the oscillation between $U$ and $\tilde U$ is defined as
	      $$
		      osc(\tilde U,  U):=\sup_{x \in \mathbb{R}^d}(\tilde U -  U) - \inf_{x \in \mathbb{R}^d}(\tilde U - U).
	      $$
	      Since $2 > \alpha > -\beta$, a Holley–Stroock perturbation argument implies that the PI constant of $p_\theta^{\sigma^\alpha}$ is comparable (up to a fixed factor) to that of the measure $\mu_d \propto \exp(-d(x)/\sigma^{2-\alpha})$ for small $\sigma$. We denote the Poincaré constant of this ideal potential as $C_{\mathrm{PI}}^{\mathrm{dist}}$.

	      A short proof for the above statement: Pick
	      $$
		      \tilde U = \log p_\theta^{\sigma^\alpha} \text{ and } U = d(x)/\sigma^{2-\alpha}.
	      $$
	      One can bound $osc(\tilde U, U)$ using Theorem 3.1 of our submission.
	      Apply the above principle to yield
	      $$
		      C_{\mathrm{PI}}^{\mathrm{TS}} \geq \exp(- O(\sigma^{\alpha+\beta})) C_{\mathrm{PI}}^{\mathrm{dist}} \geq \exp(- 1) C_{\mathrm{PI}}^{\mathrm{dist}},
	      $$
	      for a sufficiently small $\sigma$.

	\item  We note that the distance function $d(x)$ is locally Polyak–Łojasiewicz, and hence one can expect the recent results \citep{gong2024poincare} on the temperature-independent Poincaré constant for locally log-PL measure can be applied. The only requirement in \citep{gong2024poincare} that is not satisfied by $\mu_d$ is that it is not $C^2$ at $x=0$.

	\item We therefore introduce a smoothed potential
	$$
		V_c(x) := \frac{\|x\|^2}{2} + \frac{1}{2} - \sqrt{\|x\|^2 + c^2},
	$$
	and apply Holley–Stroock again to compare the PI constant of $\mu_d$ with that of $\mu_c \propto \exp( -V_c/\sigma^{2-\alpha})$. Choosing $c = \sigma^{3-\alpha}$, we can verify that $V_c$ satisfies the assumptions of the log-PL result \citep{gong2024poincare}, which implies that the corresponding Poincaré constant (denoted as $C_{\mathrm{PI}}^{\mathrm{smooth}}$) is \emph{independent of $\sigma$}.

	A short proof to bound $C_{\mathrm{PI}}^{\mathrm{dist}}$ with $C_{\mathrm{PI}}^{\mathrm{smooth}}$: Pick
	$$
		\tilde U(x) = d(x) /\sigma^{2-\alpha} \text{ and } U(x) = V_c(x)/\sigma^{2-\alpha}.
	$$
	To bound $osc(\tilde U,  U)$, notice that
	$$
		|d(x) - V_c(x)| = |\|x\| - \sqrt{\|x\|^2 + c^2}| = \frac{c^2}{\|x\| + \sqrt{\|x\|^2 + c^2}} \leq c = \sigma^{3 - \alpha}.
	$$
	Apply the perturbation principle to yield
	$$
		C_{\mathrm{PI}}^{\mathrm{dist}} \geq \exp(- O(\sigma)) C_{\mathrm{PI}}^{\mathrm{smooth}} \geq \exp(- 1) C_{\mathrm{PI}}^{\mathrm{smooth}},
	$$
	for a sufficiently small $\sigma$.

	\item  Combining these comparisons shows that the Poincaré constant of $p_\theta^{\sigma^\alpha}$, i.e., $C_{\mathrm{PI}}^{\mathrm{TS}}$, differs from $C_{\mathrm{PI}}^{\mathrm{dist}}$ and $C_{\mathrm{PI}}^{\mathrm{smooth}}$ only by a constant factor.

	\item  In this point, we discuss on proving $C_{\mathrm{PI}}^{\mathrm{smooth}}$ is independent of $\sigma$.
	      First, we note that directly apply the result in \citep{gong2024poincare} on the potential $V_c$ already yields that the Poincaré constant $C_{\mathrm{PI}}^{\mathrm{smooth}}$ is of order $\Omega(c)$: It is easy to verify the assumptions in \citep{gong2024poincare}, i.e. local PL, non-saddle point, growth condition beyond a compact set, and the boundedness of $|\Delta V_c|$, i.e. the absolute value of the Laplacian of $V_c$ within a compact set.
	      We can hence directly use Theorem 2 in \citep{gong2024poincare}. However, the quantity $|\Delta V_c|$ is of order $\frac{1}{c}$ in this vanilla analysis and hence we would yield that the Poincaré constant $C_{\mathrm{PI}}^{\mathrm{smooth}}$ is of order $\Omega(c)$.
	      It turns out that by exploiting the particular structure of $V_c$, we can further improve this result: We note that $|\Delta V_c|$ does \emph{not} need to hold in the neighborhood of the local maximum set and their analysis still goes through. We hence pick this neighborhood as a ball centered around the local maximum $x=0$ with radius $0.1$. One can see that outside of this neighborhood but within a compact set, $|\Delta V_c|$ is bounded by a $\sigma$-independent constant. Then $C_{\mathrm{PI}}^{\mathrm{smooth}}$ could be proved to be $\Omega(1)$.
	      We highlight that even the vanilla $\Omega(c)$ bound already establishes the exponential difference between $C_{\mathrm{PI}}^{\mathrm{TS}}$ (lower bounded by a polynomial in ${\sigma}$) and $C^\mathrm{LD}_{\mathrm{PI}}$ (upper bounded by exponential of $-1/{poly(\sigma)}$). Of course, the $\Omega(1)$ one leads to even bigger separation.
\end{itemize}

Putting these estimates together, we see that, at least in this unit-circle example, \emph{TS Langevin mixes strictly faster} than the original Langevin dynamics in the small-$\sigma$ regime. This illustrates that temperature-scaled Langevin is not necessarily slower—and can in fact be significantly faster—than the standard Langevin dynamics in terms of mixing time.
\subsection{A Refined Analysis for $C_{\mathrm{PI}}^{\mathrm{smooth}}$}
Directly applying the result in \citep{gong2024poincare}, we have that $C_{\mathrm{PI}}^{\mathrm{smooth}} = \Omega(\frac{1}{\sigma})$ for a sufficiently small $\sigma$.
In this subsection, we show that this can be improved to $C_{\mathrm{PI}}^{\mathrm{smooth}} = \Omega(1)$ with a small modification to the analysis of the Lyapunov function in \citep{gong2024poincare}.

\begin{proposition}\label{Thm: lyapunov method}\citep[Theorem 3.8]{AOP}
		Consider the Langevin dynamics 
		\begin{equation*}
			dX(t) = -\nabla V(X(t)) dt + \sqrt{2\epsilon}dW(t).
		\end{equation*}
        Define the associated infinitesimal generator $\mathcal{L}$ as
        \begin{equation} \label{eqn_langevin_dynamics_generator}
            \mathcal{L} := -\nabla V \cdot \nabla + \epsilon\ \Delta
        \end{equation}
        A function $\mathcal{W}:\mathbb{R}^d \rightarrow[1, \infty)$ is a \emph{Lyapunov function} for $\mathcal{L}$ if there exists $U\subseteq\mathbb{R}^d$, $b>0$, $\sigma > 0$, such that
        \begin{equation} \label{eqn:Lyapunov}
            \forall x\in\mathbb{R}^d,\ \epsilon^{-1} \mathcal{L} \mathcal{W}(x) \leq -\sigma \mathcal{W}(x) + b 1_{U}(x).
        \end{equation}
        Given the existence of such a Lyapunov function $\mathcal{W}$, if one further has that the truncated Gibbs measure $\mu_{\epsilon,U}$ satisfies PI with constant $\mathrm{PI}_{\epsilon,U} > 0$, the Gibbs measure $\mu_\epsilon$ satisfies PI with constant
        \begin{equation} \label{eqn_estimate_poincare_constant_with_subdomain}
            \rho_\epsilon \geq \frac{\sigma}{b + \rho_{\epsilon,U}}\rho_{\epsilon,U}.
        \end{equation}
        \end{proposition}
		In the context of this section, $\epsilon = \sigma^{2-\alpha}$ and $V = V_c$.
		In \citep{gong2024poincare}, the Lyapunov function is chosen to be $\mathcal{W} = \exp(\frac{V}{2\epsilon})$ and \cref{eqn:Lyapunov} can be simplified to
		\begin{equation}\label{eqn_requirement_Lyapunov}
		\frac{\mathcal{L} \mathcal{W}}{\epsilon \mathcal{W}} = \frac{\Delta {V}}{2\epsilon} - \frac{|\nabla V|^2}{4\epsilon^2} {\leq} - \sigma + b1_U.
		\end{equation}
		To establish the above inequality, \citet{gong2024poincare} partition the whole domain $\mathbb{R}^d$ into multiple disjoint parts: (1) $U$, (2) a neighborhood of the global minimum but outside of $U$, (3) neighborhoods of local maximum, (4) beyond a compact set that contains all critical points, and (5) the rest. We discuss our treatment of each subdomain.
		\begin{itemize}
			\item On (1), we follow the choice of $U$ in \citep{gong2024poincare} so the local Poincar\'e inequality there directly holds.
			
			\item On (2), i.e. in the neighborhood of the global minimum (note that under the assumptions of \citep{gong2024poincare}, all local minima are global minima), but outside of the neighborhood $U$, we follow the argument as \citep{gong2024poincare}.
			\item On (4), Beyond a compact set that contains all the local minima and maximum, we can verify that $V_c$ above fulfilles the requirements of $V$ in \citep{gong2024poincare} and hence the argument directly carries over.
			\item On (3), i.e. in a neighborhood of the local maximum, since the Laplacian is already negative, one can directly obtain \cref{eqn_requirement_Lyapunov}. Note that we will pick this neighborhood to be the ball centered at $x=0$ with radius $0.1$ for $V_c$, denoted by $\mathbb{B}(0, 0.1)$.
			\item On (5), i.e. within the said compact set, but outside of the neighborhoods of the global minimum and local maximum, \citep{gong2024poincare} requires the Laplacian to be bounded. We note that the analysis in \citep{gong2024poincare} is a bit loose and they require the boundedness to hold on the whole compact set. However, there is no need to assume the boundedness of the Laplacian on $\mathbb{B}(0, 0.1)$ as \cref{eqn_requirement_Lyapunov} is already established in (3).
		\end{itemize}
		Based on the above discussion, we notice that the global bound on the Laplacian of $V_c$ is only required within a compact set, but outside of $\mathbb{B}(0, 0.1)$, which is hence a constant independent of $\epsilon$. We hence obtain the $\Omega(1)$ bound on the Poincaré constant.

\end{document}